\documentclass{article}
\usepackage[utf8]{inputenc}

\usepackage{amsthm, amssymb, amsmath,bm}

\usepackage{algorithm} 
\usepackage{algorithmic}
\usepackage{graphicx}
\usepackage{subcaption}
\usepackage{dsfont}
\usepackage{flexisym}
\usepackage{natbib}
\usepackage{cite}
\usepackage{hyperref}
\usepackage{xcolor}
\usepackage{bbm}
\usepackage{geometry}
\theoremstyle{definition}

\newtheorem{theorem}{Theorem}[section]
\newtheorem{assumption}[theorem]{Assumption}

\newtheorem{definition}[theorem]{Definition}
\newtheorem{lemma}[theorem]{Lemma}
\newtheorem{corollary}[theorem]{Corollary}

\theoremstyle{remark}

\newtheorem{remark}{Remark}[section]

\newcommand{\mF}{\mathcal{F}}
\newcommand{\mK}{\mathcal{K}}
\newcommand{\mT}{\mathcal{T}}

\newcommand{\mC}{\mathcal{C}}
\newcommand{\mR}{\mathcal{R}}
\newcommand{\R}{\mathbb{R}}
\newcommand{\Prob}{\mathbb{P}}
\newcommand{\N}{\mathbb{N}}
\newcommand{\E}{\mathbb{E}}
\newcommand{\indfunc}{\mathbb{I}}

\def\E{\mathbb{E}}

\title{A Near-Optimal Change-Detection Based Algorithm for Piecewise-Stationary Combinatorial Semi-Bandits}
\author{
Huozhi Zhou\textsuperscript{1}\thanks{The first two authors contributed equally to this work.},
Lingda Wang\textsuperscript{1}\footnotemark[1],
Lav R. Varshney\textsuperscript{1},
Ee-Peng Lim\textsuperscript{2}\\
\textsuperscript{1}{ECE Department, UIUC}\\
\textsuperscript{2}{School of Information Systems, SMU}\\
\{hzhou35, lingdaw2, varshney\}@illinois.edu,
eplim@smu.edu.sg}

\begin{document}
\maketitle
\begin{abstract}
\label{sec:abstract}
We investigate the piecewise-stationary combinatorial semi-bandit problem. Compared to the original combinatorial semi-bandit problem, our setting assumes the reward distributions of base arms may change in a piecewise-stationary manner at unknown time steps. We propose an algorithm, \texttt{GLR-CUCB}, which incorporates an efficient combinatorial semi-bandit algorithm, \texttt{CUCB}, with an almost parameter-free change-point detector, the \emph{Generalized Likelihood Ratio Test} (GLRT).  Our analysis shows that the regret of \texttt{GLR-CUCB} is upper bounded by $\mathcal{O}(\sqrt{NKT\log{T}})$, where $N$ is the number of piecewise-stationary segments, $K$ is the number of base arms, and $T$ is the number of time steps. As a complement, we also derive a nearly matching regret lower bound on the order of $\Omega(\sqrt{NKT}$), for both piecewise-stationary multi-armed bandits and combinatorial semi-bandits, using information-theoretic techniques and judiciously constructed piecewise-stationary bandit instances. Our lower bound is tighter than the best available regret lower bound, which is $\Omega(\sqrt{T})$. Numerical experiments on both synthetic and real-world datasets demonstrate the superiority of \texttt{GLR-CUCB} compared to other state-of-the-art algorithms.
\end{abstract}

\section{Introduction}
\label{sec:intro}
The multi-armed bandit (MAB) problem, first proposed by~\citet{thompson1933likelihood},  has been studied extensively in statistics and machine learning communities, as it models many online decision making problems such as online recommendation~\citep{li2016collaborative}, computational advertising~\citep{tang2013automatic}, and crowdsourcing task allocation~\citep{hassan2014multi}. The classical MAB is modeled as an agent repeatedly pulling one of $K$ arms and observing the reward generated, with the goal of minimizing the \textit{regret} which is the difference between the reward of the optimal arm in hindsight and the reward of the arm chosen by the agent. This classical problem is well understood for both stochastic~\citep{lai1985asymptotically} and adversarial settings~\citep{auer2002nonstochastic}. The stochastic setting is when the reward of each arm is generated from a fixed distribution, and it is well known that the problem-dependent lower bound is of order $\Omega(\log T)$~\citep{lai1985asymptotically}, where $T$ is the number of time steps. Several algorithms have been proposed and proven to achieve $\mathcal{O}(\log T)$ regret~\citep{agrawal2012analysis,auer2002finite}. The adversarial setting is when at each time step the environment generates the reward in an adversarial manner, whose minimax regret lower bound is of order $\Omega(\sqrt{T})$. Several algorithms achieving order-optimal (up to poly-logarithm factors) regret have been proposed in recent years~\citep{hazan2011better,bubeck2018sparsity,li2019bandit}. 

Many real-world applications, however, have a combinatorial nature that cannot be fully characterized by the classical MAB model. For example, online movie sites aim to recommend multiple movies to the users to maximize their utility under some constraints (e.g. recommend at most one movie for each category). This phenomenon motivates the study of combinatorial semi-bandits (CMAB), which aims to identify the best superarm, a set of base arms with highest aggregated reward. Several algorithms for stochastic CMAB with provable guarantees based on optimism principle have been proposed recently~\citep{chen2013combinatorial,kveton2015tight,combes2015combinatorial}. All of these algorithms use some oracle to overcome the curse of dimension of the action space for solving some combinatorial optimization problem at each iteration. In addition, these algorithms achieve the optimal regret upper bound which is $\mathcal{O}(C\log T)$, where $C$ is an instance-dependent parameter. Adversarial CMAB is also well studied. Several algorithms achieve optimal regret of order $\mathcal{O}(\sqrt{T})$ based on either \emph{Follow-the-Regularized-Leader} or \emph{Follow-the-Perturbed-Leader}, both of which are general frameworks for adversarial online learning algorithm design~\citep{hazan2016introduction}. Moreover, one recent study has developed an algorithm that is order-optimal for both stochastic and adversarial CMAB~\citep{zimmert2019beating}. 

Although both stochastic and adversarial CMAB are well-studied, understanding of the scenario lying in the ``middle'' of these two settings is still limited. Such a ``middle'' setting where the reward distributions of base arms slowly change over time may be a more realistic model in many applications. For instance, in online recommendation systems, users' preference are unlikely to be either time-invariant or to change significantly and frequently over time. Thus in this case it would be too ideal to assume the stochastic CMAB model and too conservative to assume the adversarial CMAB model. Similar situations appear in web search, online advertisement, and crowdsourcing~\citep{yu2009piecewise,pereira2018analyzing,vempaty2014reliable}. As such, we investigate a setting lying between these two standard CMAB models, namely piecewise-stationary combinatorial semi-bandit, which we will define formally in Section~\ref{sec:pf}. Piecewise-stationary CMAB is a natural generalization of the piecewise-stationary MAB model~\citep{hartland2007change,kocsis2006discounted,garivier2011upper}, and can be interpreted as an approximation to the slow-varying CMAB problem. Roughly, compared to the stochastic CMAB, we assume reward distributions of base arms remain fixed for certain time periods called \textit{piecewise-stationary segments}, but can change abruptly at some unknown time steps, called \textit{change-points}. 

Previous works on piecewise-stationary MAB may be divided into two categories: \textit{passively adaptive approach}~\citep{garivier2011upper,besbes2014stochastic,wei2018abruptly} and \textit{actively adaptive approach}~\citep{cao2019nearly,liu2018change,besson2019generalized,auer2019adaptively}. Passively adaptive approaches make decisions based on the most recent observations and are unaware of the underlying distribution changes. On the contrary, actively adaptive approaches incorporate a change-point detector subroutine to monitor the reward distributions, and restart the algorithm once a change-point is detected. Numerous empirical experiments have shown that actively adaptive approaches outperform passively adaptive approaches~\citep{mellor2013thompson}, which motivates us to adopt an actively adaptive approach.

Our main contributions include the following:
\begin{itemize}
    \item We propose a simple and general algorithm for piecewise-stationary CMAB, named \texttt{GLR-CUCB}, which is based on \texttt{CUCB}~\citep{chen2013combinatorial} with a novel change-point detector, the \emph{generalized likelihood ratio test} (GLRT) ~\citep{besson2019generalized}. The advantage of GLRT change-point detection is it is almost parameter-free and thus easy to tune compared to previously proposed change-point detection methods used in nonstationary MAB, such as \texttt{CUSUM}~\citep{liu2018change} and \texttt{SW}~\citep{cao2019nearly}.
    \item For any combinatorial action set, we derive the problem-dependent regret bound for \texttt{GLR-CUCB} under mild conditions (see Section~\ref{sec:reg_ub}). When the number of change points $N$ is known beforehand, the regret of \texttt{GLR-CUCB} is upper bounded by $\mathcal{O}(C_1NK^2\log{T}+C_2\sqrt{NKT\log{T}})$ (nearly order-optimal within poly-logarithm factor in $T$), where $K$ is number of base arms. When $N$ is unknown, the algorithm achieves $\mathcal{O}(C_1NK^2\log{T}+C_2N\sqrt{KT\log {T}})$. Here, $C_1$ and $C_2$ are problem-dependent constants which do not depend on $T$, $N$, or $K$.
    \item We derive a tighter minimax lower bound for both piecewise-stationary MAB and  piecewise-stationary CMAB on the order of $\Omega(\sqrt{NKT})$. Since piecewise-stationary MAB is a special instance of piecewise-stationary CMAB in which every superarm is a single arm, thus any minimax lower bound holds for piecewise-stationary MAB also holds for piecewise-stationary CMAB. To the best of our knowledge, this is the best existing minimax lower bound for piecewise-stationary CMAB. Previously, the best available lower bound is $\Omega(\sqrt{T})$~\citep{garivier2011upper}, which does not depend on $N$ or $K$.
    \item We demonstrate that \texttt{GLR-CUCB} performs significantly better than state-of-the-art algorithms through experiments on both synthetic and real-world datasets.
\end{itemize}

The remainder of this paper is organized as follows: the formal problem formulation and some preliminaries are introduced in Section~\ref{sec:pf}, then the proposed \texttt{GLR-CUCB} algorithm in Section~\ref{sec:alg}. We derive the upper bound on the regret of our algorithm in Section~\ref{sec:reg_ub}, and the minimax regret lower bound in Section~\ref{sec:reg_lb}. Section~\ref{sec:exp} gives our experiment results. Finally, we conclude the paper. Due to the page limitation, we postpone proofs and additional experimental results to the appendix.

\section{Problem Formulation and Background}
\label{sec:pf}
In this section, we start with the formal definition of piecewise-stationary combinatorial semi-bandit as well as some technical assumptions in Section~\ref{subsec:pf_bandit}. Then, we introduce the GLR change-point detector used in our algorithm design, and its advantage in Section~\ref{subsec:def_glr}.
\subsection{Piecewise-Stationary Combinatorial Semi-Bandits}
\label{subsec:pf_bandit}
A piecewise-stationary combinatorial semi-bandit is characterized by a tuple $(\mK,\mF,\mT,\{f_{k,t}\}_{k\in\mK,t\in\mT},r_{\bm{\mu}_t}\left(S_t\right))$. Here, $\mK=\{1,\ldots,K\}$ is the set of $K$ base arms; $\mF\subseteq 2^{\mK}$ is the set of all super arms; $\mT=\{1,\ldots,T\}$ is a sequence of $T$ time steps; $f_{k,t}$ is the reward distribution of arm $k$ at time $t$ with mean $\mu_{k,t}$ and bounded support within $[0,1]$; $r_{\bm{\mu}_t}\left(S_t\right):\mF\times[0,1]^K\mapsto\R$ is the expected reward function defined on the super arm $S_t$ and mean vector of all base arms $\bm{\mu}_t := [\mu_{1,t},\mu_{2,t},\ldots,\mu_{K,t}]^\top$ at time $t$. Like~\citet{chen2013combinatorial}, we assume the expected reward function $r_{\bm{\mu}}\left(S\right)$ satisfies the following two properties:
\begin{assumption}[Monotonicity]
\label{assump:mon}
Given two arbitrary mean vectors $\bm{\mu}$ and $\bm{\mu}'$, if $\mu_k\ge\mu'_k,\ \forall k\in\mK$, then $r_{\bm{\mu}}\left(S\right)\ge r_{\bm{\mu}'}\left(S\right)$.
\end{assumption}
\begin{assumption}[$L$-Lipschitz]
\label{assump:lip}
Given two arbitrary mean vectors $\bm{\mu}$ and $\bm{\mu}'$, there exists an $L<\infty$ such that $\left|r_{\bm{\mu}}\left(S\right)- r_{\bm{\mu}'}\left(S\right)\right|\le L\|\mathcal{P}_S(\bm{\mu}-\bm{\mu}')\|_2$, $\forall S\in\mF$, where $\mathcal{P}_S(\cdot)$ is the projection operator specified as $\mathcal{P}_S(\bm{\mu})=\left[\mu_1\indfunc\{1\in S\},\ldots,\mu_k\indfunc\{k\in S\},\ldots,\mu_K\indfunc\{K\in S\}\right]^\top$ in terms of the indicator function $\indfunc\{\cdot\}$.
\end{assumption}

In the piecewise i.i.d. model, we define $N$, the number of piecewise-stationary segments in the reward process, to be
\begin{equation*}
    N=1+\sum_{t=1}^{T-1}\indfunc{\{\exists k\in\mK~\mbox{s.t.}~f_{k,t}\neq f_{k,t+1}\}}.
\end{equation*}
 We denote these $N-1$ change-points as $\nu_1,\nu_2,\ldots,\nu_{N-1}$ respectively, and we let $\nu_0=0$ and $\nu_N=T$. For each piecewise-stationary segment $t\in[\nu_{i-1}+1,\nu_i]$, we use $f_{k}^i$ and $\mu_k^i$ to denote the reward distribution and the expected reward of arm $k$ on the $i$th piecewise-stationary segment, respectively. The vector encoding the expected rewards of all base arms at the $i$th segment is denoted as $\bm{\mu}^i=[\mu_1^i,\cdots,\mu_K^i]^\top$, $\forall i=1,\cdots,N$. Note that when a change-point occurs, there must be at least one arm whose reward distribution has changed, however, the rewards distributions of all base arms do not necessarily change.

For a piecewice-stationary combinatorial semi-bandit problem, at each time step $t$, the learning agent chooses  a super arm $S_t\in\mF$ to play based on the rewards observed up to time $t$. When the agent plays a super arm $S_t$, the reward $\{X_{I_t}\}_{I_t\in S_{t}}$ of base arms contained in super arm $S_t$ are revealed to the agent and the reward of super arm $R_t(S_{t})$ as well. We assume that the agent has access to an $\alpha$-approximation oracle, to carry out combinatorial optimization, defined as follows.
\begin{assumption}[$\alpha$-approximation oracle]
Given a mean vector $\bm{\mu}$, the $\alpha$-approximation oracle $\mbox{Oracle}_\alpha(\bm{\mu})$ outputs an $\alpha$-suboptimal super arm $S$ such that $r_{\bm{\mu}}\left(S\right)\ge\alpha\max_{S\in\mF}r_{\bm{\mu}}\left(S\right)$.
\end{assumption}
\begin{remark}
 The approximation oracle assumption was first proposed in~\citet{chen2013combinatorial} for the combinatorial semi-bandit setting. This assumption is reasonable since many combinatorial NP-hard problems admit approximation algorithms, which can be solved efficiently in polynomial time~\citep{ausiello1995approximate}. There are also many combinatorial problems which are not NP hard and can be solved efficiently. One example is the top-$m$ arm identification problem in the bandit setting~\citep{cao2015top}, where any efficient sorting algorithm suffices. 
\end{remark}
 As only an $\alpha$-approximation oracle is used for optimization, it is reasonable to use expected $\alpha$-approximation cumulative regret to measure the performance of the learning agent, defined as follows.
\begin{definition}[Expected $\alpha$-approximation cumulative regret]
 The agent's policy is evaluated by its expected $\alpha$-approximation cumulative regret,
 \begin{equation*}
    \mR(T)=\E\left[\alpha\sum_{t=1}^T\max_{S\in\mF}r_{\bm{\mu}_t}(S)-\sum_{t=1}^Tr_{\bm{\mu}_t}(S_{t})\right],
\end{equation*}
where the expectation $\E[\cdot]$ is taken with respect to the selection of $\{S_t|S_t\in\mathcal{F}\}$.
\end{definition}

\subsection{Generalized Likelihood Ratio Change-Point Detector for Sub-Bernoulli Distribution}
\label{subsec:def_glr}
Sequential change-point detection is a classical problem in statistical sequential analysis, but most existing works make additional assumptions on the pre-change and post-change distributions which might not hold in the bandit setting~\citep{siegmund2013sequential,basseville1993detection}. In general, designing algorithms with provable guarantees for change detection with little assumption on pre-change and post-change distributions is very challenging. In our algorithm design, we will use the generalized likelihood ratio (GLR) change-point detector~\citep{besson2019generalized}, which works for any sub-Bernoulli distribution. Compared to other existing change detection methods used in piecewise-stationary MAB, the GLR detector has less parameters to be tuned and needs less prior knowledge for the bandit instance. Specifically, GLR only needs to tune the threshold $\delta$ for the change-point detection, and does not require the smallest change in expectation among all change-points. On the contrary, CUSUM~\citep{liu2018change} and SW~\citep{cao2019nearly} both need more parameters to be tuned and need to know the smallest magnitude among all change-points beforehand, which limits their practicality. 

To define GLR change-point detector, we need some more definitions for clarity. A distribution $f$ is said to be sub-Bernoulli if $\E_{X\sim f}[e^{\lambda X}]\le e^{\phi_\mu(\lambda)}$, where $\mu=\E_{X\sim f}[X]$; $\phi_\mu(\lambda)=\log(1-\mu-\mu e^{\lambda})$ is the log moment generating function of a Bernoulli distribution with mean $\mu$. Notice that the support of reward distribution $f_{k,t}$, $\forall k\in\mK,\,t\in\mT$ is a subset of the interval $[0,1]$, thus all $\{f_{k,t}\}$ are sub-Bernoulli distributions with mean $\{\mu_{k,t}\}$, due to the following lemma. 
\begin{lemma}[Lemma 1 in~\citet{cappe2013kullback}]
Any distribution $f$ with bounded support within the interval $[0,1]$ is a sub-Bernoulli distribution that satisfies:
\begin{align*}
    \E_{X\sim f}[e^{\lambda X}]\le e^{\phi_\mu(\lambda)}.
\end{align*}
\end{lemma}

Suppose we have either a time sequence $\{X_t\}_{t=1}^n$ drawn from a sub-Bernoulli distribution for any $t\le n$ or two sub-Bernoulli distributions with an unknown change-point $s\in[1,n-1]$. This problem can be formulated as a parametric sequential test:
\begin{align*}
\mathcal{H}_0&:\exists f_0:X_1,\ldots,X_n\overset{i.i.d.}{\sim}~f_0,\\
\mathcal{H}_1&:\exists f_0\neq f_1,\, \tau\in[1,n-1]:\,X_1,\ldots,X_\tau\overset{i.i.d.}{\sim}f_0~\mbox{and}~X_{\tau+1},\ldots,X_{n}\overset{i.i.d.}{\sim}f_1.
\end{align*}
The GLR statistic for sub-Bernoulli distributions is:
\begin{align}
\label{eq:GLR_statistic}
    \text{GLR}(n)=\sup_{s\in[1,n-1]}\quad[s\times\text{kl}(\hat{\mu}_{1:s},\hat{\mu}_{1:n})+(n-s)\times\text{kl}\left(\hat{\mu}_{s+1:n},\hat{\mu}_{1:n}\right)],
\end{align}
where $\hat{\mu}_{s:s'}$ is the mean of the observations collected between $s$ and $s'$, and $\mbox{kl}(x,y)$ is the binary relative entropy between Bernoulli distributions,
\begin{equation*}
    \mbox{kl}(x,y) = x\log\left(\frac{x}{y}\right)+(1-x)\log\left(\frac{1-x}{1-y}\right).
\end{equation*}
If the GLR in Eq.~\eqref{eq:GLR_statistic} is large, it indicates that hypothesis $\mathcal{H}_1$ is more likely. Now, we are ready to define the sub-Bernoulli GLR change-point detector with confidence level $\delta\in(0,1)$.
\begin{definition}
The sub-Bernoulli GLR change-point detector with threshold function $\beta(n,\delta)$ is
\begin{align*}
    \tau:=\inf\{n\in\mathbb{N}:\sup_{s\in[1,n-1]}\quad[s\times\mbox{kl}(\hat{\mu}_{1:s},\hat{\mu}_{1:n})+(n-s)\times\mbox{kl}(\hat{\mu}_{s+1:n},\hat{\mu}_{1:n})]\ge\beta(n,\delta)\},
\end{align*}
where $\beta(n,\delta)=2\mathcal{Q}\left(\frac{\log(3n\sqrt{n}/\delta)}{2}\right)+6\log(1+\log(n))$, and $\mathcal{Q}(\cdot)$ is as in Eq. (13) in~\citet{kaufmann2018mixture}.
\end{definition} 
The pseudo-code of sub-Bernoulli GLR change-point detector is summarized in Algorithm~\ref{alg:cd} for completeness.

\begin{algorithm}
\caption{Sub-Bernoulli GLR Change-Point Detector: $\mbox{GLR}(X_1,\cdots, X_n;\delta)$}
\label{alg:cd}
\begin{algorithmic}[1]

\REQUIRE observations $X_1, \ldots, X_n$ and confidence level $\delta$.
\IF{$\sup_{s\in[1,n-1]}\quad[s\times\mbox{kl}(\hat{\mu}_{1:s},\hat{\mu}_{1:n})+(n-s)\times\mbox{kl}(\hat{\mu}_{s+1:n},\hat{\mu}_{1:n})]\ge\beta(n,\delta)$} 
\STATE Return \textbf{True}
\ELSE 
\STATE Return \textbf{False}
\ENDIF

\end{algorithmic}
\end{algorithm}

\section{The \texttt{GLR-CUCB} Algorithm}
\label{sec:alg}
% \iffalse
% Our proposed algorithm, C-sBGLRT-UCB, combines the a combinatorial bandit algorithm with a change point detector running on each arm. Specifically, we also have a forced exploration with prob $p$ to ensure each arm is sampled enough and the change point  of under-sampled arms can be detected. And a restart happen when the sBGLRT detect a change. To be specific, $\text{sBGLRT}\left(X_{I_t,1},\cdots,X_{I_t,n_{I_t}};\delta\right)=\text{True}$ if and only if
% \begin{equation}
%     \sup_{s\in[1,n_{I_t}]}\left[s\times\text{kl}\left(\frac{1}{s}\sum_{n=1}^sZ_{I_t,n},\frac{1}{n_{I_t}}\sum_{n=1}^{n_{I_t}}Z_{I_t,n}\right)+(n_{I_t}-s)\times\text{kl}\left(\frac{1}{n_{I_t}-s}\sum_{n=s+1}^{n_{I_t}}Z_{I_t,n},\frac{1}{n_{I_t}}\sum_{n=1}^{n_{I_t}}Z_{I_t,n}\right)\right]\ge\beta(n_{I_t},\delta).
% \end{equation}
% And the UCB index are defined by
% \begin{equation}
%     \text{UCB}(I_t)\leftarrow\frac{1}{n_{I_t}}\sum_{n=1}^{n_{I_t}}Z_{I_t,n}+\sqrt{\frac{3\ln(t-\tau)}{2n_{I_t}}}.
% \end{equation}
% Here $\tau$ denotes the last detection time and $n_k$ denotes the number of observations of $k$-th arm after $\tau$.
% \fi
Our proposed algorithm, \texttt{GLR-CUCB}, incorporates an efficient combinatorial semi-bandit problem algorithm \texttt{CUCB}~\citep{chen2013combinatorial} with a change-point detector running on each base arm (See Algorithm~\ref{alg:bandit}). The \texttt{GLR-CUCB} requires the number of time steps $T$, the number of base arms $K$, uniform exploration probability $p$, and the confidence level $\delta$ as inputs. Let $\tau$ denote the last change-point detection time, and $n_k$ denote the number of observations of base arm $k$ after $\tau$, which are initialized as zeros at the beginning of the algorithm. 

At each time step, the \texttt{GLR-CUCB} first determines if it will enter forced uniform exploration (to ensure each base arm collects sufficient samples for the change-point detection) according to the condition in line 3. If it is in a forced exploration, a random super arm $S_t$ that contains $a$ (line 4) is played, to ensure sufficient number of samples are collected for every base arm. Otherwise, the next super arm $S_t$ to be played is determined by the $\alpha$-approximation oracle $\mbox{Oracle}_\alpha(\cdot)$ (line 6) given the UCB indices (line 19). Then, the learning agent plays the super arm $S_t$, and gets the reward $R_t\left(S_t\right)$ of the super arm $S_t$ and the rewards $\{X_{I_t}\}_{I_t\in S_t}$ of the base arm $I_t$'s that are contained in the super arm $S_t$ (line 8). In the next step, the algorithm updates the statistics for each base arm (lines 10-11) in order to run the GLR change-point detector (Algorithm~\ref{alg:cd}) with confidence level $\delta$ (line 12). If the GLR change-point detector detects a change in distribution for any of the base arms, the algorithm sets $\tau$ to be the current time step and all $n_k$'s to be 0 (line 13-14) before going into time step $t+1$. Lastly, the UCB indices of all base arms are updated (line 19).

\begin{algorithm}[ht]
    \caption{The \texttt{GLR-CUCB} Algorithm}
    \label{alg:bandit}
    \begin{algorithmic}[1]
        \REQUIRE $T\in\N$, $K\in\N$, exploration probability $p\in(0,1)$, confidence level $\delta\in(0,1)$.

        \STATE \textbf{Initialization: } $\forall k \in \mK$, $n_k \leftarrow 0$; $\tau\leftarrow 0$.
        \FORALL{$t=1,2,\ldots, T$}
            \IF{$p>0$ and $a\leftarrow (t-\tau) \mod \left\lfloor \frac{K}{p}\right\rfloor \in \mK$}
                \STATE Randomly choose $S_{t}$ with $a\in S_{t}$. 
            \ELSE{
                \STATE $S_{t}=\mbox{Oracle}_\alpha\left(\mbox{UCB}\right)$.
            }
            \ENDIF
            \STATE Play super arm $S_{t}$ and get the reward $R_{t}(S_{t})$ and $X_{I_t,t},\,\forall I_t\in S_{t}$.
            \FORALL{$I_{t}\in S_{t}$}
                \STATE $Z_{I_t,n_{I_t}}\leftarrow X_{I_t,t}$.
                \STATE $n_{I_t}\leftarrow n_{I_t}+1$.
                \IF{$\mbox{GLR}\left(Z_{I_t,1},\cdots,Z_{I_t,n_{I_t}};\delta\right)=\text{True}$}
                    \STATE $n_k\leftarrow0,\forall k\in\mK$.
                    \STATE $\tau\leftarrow t$.
                \ENDIF
            \ENDFOR
            
            \FORALL{$k=1,\cdots,K$}
                \IF{$n_k\neq 0$}
                    \STATE $\text{UCB}(k)\leftarrow\frac{1}{n_\ell}\sum_{n=1}^{n_k}Z_{k,n}+\sqrt{\frac{3\log(t-\tau)}{2n_k}}$.
                \ENDIF
            \ENDFOR
            
        \ENDFOR
    \end{algorithmic}
\end{algorithm}

\begin{remark}
The uniform exploration is necessary for this algorithm, and similar strategy has been adopted in~\citet{liu2018change, cao2019nearly}. Intuitively, uniform exploration ensures each base arm gathers sufficient samples to guarantee quick change detection whereas pure UCB exploration is incapable of this. One more rigorous argument is given in~\citet{garivier2011upper}, which shows that theoretically pure UCB exploration performs badly on piecewise-stationary MAB.
\end{remark}

\begin{remark}
Thompson sampling (TS) often performs better than UCB policy in empirical simulations, but it has been shown that one cannot incorporate an approximate oracle in TS for even MAB problems~\citep{wang2018thompson}. Thus our algorithm adopts UCB policy for the bandit component to ensure compatibility with approximation oracle. 
\end{remark}

\section{Regret Upper Bound}
\label{sec:reg_ub}
In this section, we analyze the $T$-step regret of our proposed algorithm \texttt{GLR-CUCB}. Recall $T$ is the time horizon, $N$ is the number of piecewise-stationary segments, $\nu_1, \ldots, \nu_{N-1}$ are the change-points, and for each segment $i\in[N]$, $\bm{\mu}^{i}\in\mathbb{R}^{K}$ is the vector encoding the expected rewards of all base arms. A super arm $S$ is \textit{bad} with respect to the $i$th piecewise-stationary distributions if $r_{\bm{\mu}^i}(S)\leq\alpha\max_{\tilde{S}\in\mathcal{F}}r_{\bm{\mu}^i}(\tilde{S})$. We define $\mathcal{S}_B^i=\{S|r_{\bm{\mu}^i}(S)\leq \alpha\max_{\tilde{S}\in\mathcal{F}}r_{\bm{\mu}^i}(\tilde{S})\}$ to be the set of bad super arms with respect to the $i$th piecewise-stationary segment. We define the \textit{suboptimality gap} in the $i$th stationary segment as follows:
\begin{align*}
    \Delta^{\text{min},i}_{\textup{opt}} &= \alpha\max_{\tilde{S}\in\mathcal{F}}r_{\bm{\mu}^i}(\tilde{S})-\max\{r_{\bm{\mu}^i}(S)|S\in\mathcal{S}_B^i\},\\
      \Delta^{\text{max},i}_{\textup{opt}} &= \alpha\max_{\tilde{S}\in\mathcal{F}}r_{\bm{\mu}^i}(\tilde{S})-\min\{r_{\bm{\mu}^i}(S)|S\in\mathcal{S}_B^i\}.
\end{align*}
Furthermore, let $\Delta^{\text{max}}_{\textup{opt}}=\max_{i\in[N]}\Delta^{\text{max}, i}_{\textup{opt}}$ and $\Delta^{\text{min}}_{\textup{opt}}=\min_{i\in[N]}\Delta^{\text{min}, i}_{\textup{opt}}$ be the maximum and minimum sub-optimal gaps for the whole time horizon, respectively. Lastly, denote the largest gap at change-point $\nu_i$ as $\Delta_{\textup{change}}^i=\max_{k\in\mK}\left|\mu_k^{i+1}-\mu_k^i\right|$, $\forall 1\le i\le N-1$, and $\Delta^0_{\textup{change}}=\max_{k\in\mK}\left|\mu_k^1\right|$. We need the following assumption for our theoretical analysis. 
\begin{assumption}
\label{assump:gap}
Define $d_i=d_i(p,\delta)=\left\lceil \{4K/p\left(\Delta_{\textup{change}}^i\right)^2\}\beta(T,\delta)+\frac{K}{p}\right\rceil$ and assume $\nu_i-\nu_{i-1}\ge2\max\{d_i,d_{i-1}\}$, $\forall i=1,\ldots,N-1$, where $\nu_{N}-\nu_{N-1}\ge 2d_{N-1}$.
\end{assumption}
Tuning $\delta$ and $p$ properly (See Corollary~\ref{col:regret_tuning}), and applying the upper bound on $\mathcal{Q}(x)$ by~\citet{kaufmann2018mixture} with $x\ge 5$,
\begin{align*}
\label{eq:G_up}
    \mathcal{Q}(x)\le x+4\log(1+x+\sqrt{2x}),
\end{align*}
the length of each piecewise-stationary segment is $\Omega(\sqrt{T\log{T}})$.
Roughly, we assume the length of each stationary segment to be sufficiently long, in order to let the GLR change-point detector detect the change in distribution within a reasonable delay with high probability. Similar assumption on the length of stationary segments also appears in other literature on piecewise stationary MAB~\citep{liu2018change,cao2019nearly,besson2019generalized}. Note that Assumption~\ref{assump:gap} is only required for the theoretical analysis; Algorithm~\ref{alg:bandit} can be implemented regardless of this assumption. Now we are ready to state the regret upper bound for Algorithm~\ref{alg:bandit}. 
\begin{theorem}
\label{thm:regret_ub}
Running \texttt{GLR-CUCB} with Assumptions~\ref{assump:mon},~\ref{assump:lip}, and~\ref{assump:gap}, the expected $\alpha$-approximation cumulative regret of \texttt{GLR-CUCB} with exploration probability $p$ and confidence level $\delta$ satisfies
\begin{equation*}
    \mR(T)\le\underbrace{\sum_{i=1}^N \widetilde{C}_i}_{(a)}+\underbrace{\Delta^{\text{max}}_{\textup{opt}}Tp}_{(b)}+\underbrace{\sum_{i=1}^{N-1}\Delta^{\text{max},i+1}_{\textup{opt}}d_i}_{(c)}+\underbrace{3NT\Delta^{\text{max}}_{\textup{opt}}K\delta}_{(d)},
\end{equation*}
where $\widetilde{C}_i=\left(6L^2K^2\log{T}/\left(\Delta^{\text{min},i}_{\textup{opt}}\right)^2+\pi^2/6+K\right)\Delta^{\text{max},i}_{\textup{opt}}$.
\end{theorem}
Theorem~\ref{thm:regret_ub} indicates that the regret comes from four sources. Terms (a) and (b) correspond to the cost of exploration, while terms (c) and (d) correspond to the cost of change-point detection. More specifically, term (a) is due to UCB exploration, term (b) is due to uniform exploration, term (c) is due to expected delay of GLR change-point detector, and term (d) is due to the false alarm probability of GLR change-point detector. We need to carefully tune the exploration probability $p$ and false alarm probability $\delta$ to balance the trade-off.

The following corollary comes directly from Theorem~\ref{thm:regret_ub} by properly tuning the parameters in the algorithm.
\begin{corollary}
\label{col:regret_tuning}
Let $\Delta^{\text{min}}_{\textup{change}}=\min_{i\in[N-1]}\Delta_{\textup{change}}^i$, we have
\begin{enumerate}
    \item ($N$ is known) Choosing $\delta=\frac{1}{T}$, $p=\sqrt{\frac{NK\log{T}}{T}}$, gives $\mR(T)=\mathcal{O}\left(\frac{NK^2\log{T}\Delta^{\text{max}}_{\textup{opt}}}{\left(\Delta^{\text{min}}_{\textup{opt}}\right)^2}+\frac{\sqrt{NKT\log{T}}\Delta^{\text{max}}_{\textup{opt}}}{\left(\Delta^{\text{min}}_{\textup{change}}\right)^2}\right)$;
    \item ($N$ is unknown) Choosing $\delta=\frac{1}{T}$, $p=\sqrt{\frac{K\log{T}}{T}}$, gives $\mR(T)=\mathcal{O}\left(\frac{NK^2\log{T}\Delta^{\text{max}}_{\textup{opt}}}{\left(\Delta^{\text{min}}_{\textup{opt}}\right)^2}+\frac{N\sqrt{KT\log{T}}\Delta^{\text{max}}_{\textup{opt}}}{\left(\Delta^{\text{min}}_{\textup{change}}\right)^2}\right)$.
\end{enumerate}
\end{corollary}
\begin{remark}
The effect of oracle is to reduce the dependency on number of base arms $K$ from $|\mathcal{F}|$ to $K^2$ during exploration (first term in the regret appeared in Corollary~\ref{col:regret_tuning}). In the worst case, $|\mathcal{F}|$ can be exponential with respect to $K$. Recall that if we use standard MAB algorithms for exploration, the dependency on $K$ is $|\mathcal{F}|$.
\end{remark}
\begin{remark}
As $T$ becomes larger, the regret is dominated by the cost of change-point detection, which has similar order compared to the regret bound of piecewise-stationary MAB algorithms. This is reasonable since our setting assumes that we have access to the reward of the base arms contained in the super arm played by the agent.
\end{remark}
\begin{remark}
When $T$ is large, the order of the regret bound is similar to the regret bound of adversarial bandit. But note that regret definition for adversarial bandit and piecewise-stationary bandit is different. For the first case, the regret is evaluated with respect to one fixed arm which is optimal for the whole horizon. But for the second case, the regret is evaluated with respect to point-wise optimal arm, which is much more challenging. 
\end{remark}
We can use Corollary~\ref{col:regret_tuning} as a guide for parameter tuning. The above corollary indicates that without knowledge of the number of change-points $N$, we pay a penalty of a factor of $\sqrt{N}$ in the long run. 

For the detailed proof of Theorem~\ref{thm:regret_ub}, see Appendix~\ref{sec:proof_regret_ub}. Here we sketch the proof; to do so we need some additional lemmas. We start by proving the regret of \texttt{GLR-CUCB} under the stationary scenario. 
\begin{lemma}
\label{lem:st_reg}
Under the stationary scenario, i.e. $N=1$, the \textit{$\alpha$-approximation cumulative regret} of \texttt{GLR-CUCB} is upper bounded as:
\begin{align*}
    \mR(T)\leq \Delta^{\text{max}, 1}_{\textup{opt}}T\mathbb{P}(\tau_1\leq T) + \Delta^{\text{max}, 1}_{\textup{opt}}Tp +\widetilde{C}_1.
\end{align*}
\end{lemma}
The first term is due to the possible false alarm of the change-point detection subroutine, the second term is due to the uniform exploration, and the last term is due to the UCB exploration. We upper bound the false alarm probability in Lemma~\ref{lem:false_alarm}, as follows.
\begin{lemma}[False alarm probability in the stationary scenario]
\label{lem:false_alarm}
Consider the stationary scenario, i.e. $N=1$, with confidence level $\delta>0$; we have that
\begin{equation*}
    \Prob\left(\tau_1\le T\right)\le K\delta.
\end{equation*}
\end{lemma}
\begin{remark}
By setting $\delta=\frac{1}{T}$, we will have $\Prob\left(\tau_1\le T\right)\le \frac{K}{T}$. Asymptotically, the false alarm probability will go to 0.
\end{remark}
In the next lemma, we show the GLR change-point detector is able to detect change in distribution reasonably well with high probability, given all previous change-points were detected reasonably well. The formal statement is as follows.  
\begin{lemma}
\label{lem:cond_prob}
(Lemma 12 in~\citet{besson2019generalized}) Define the event $\mC^{(i)}$ that all the change-points up to $i$th one have been detected successfully within a small delay:
\begin{equation}
\label{eq:C_event}
    \mC^{(i)}=\left\{\forall j\le i,\tau_j\in\left\{\nu_j+1,\cdots,\nu_j+d_j\right\}\right\}.
\end{equation} 
Then, $\Prob[\tau_i\le\nu_i|\mC^{(i-1)}]\le K\delta$, and $\Prob[\tau_i\ge\nu_i+d_i|\mC^{(i-1)}]\le \delta$,
where $\tau_i$ is the detection time of the $i$th change-point.
\end{lemma}
Lemma~\ref{lem:cond_prob} provides an upper bound for the conditional expected detection delay, given the good events $\{\mC^{(i)}\}$.
\begin{corollary}[Bounded conditional expected delay]
\label{col:exp_delay}
$\E\left[\tau_i-\nu_i|\mC^{(i)}\right]\le d_i$.
\end{corollary}
Given these lemmas, we can derive the regret upper bound for \texttt{GLR-CUCB} in a recursive manner. Specifically, we prove Theorem~\ref{thm:regret_ub} by recursively decomposing the regret into a collection of good events and bad events. The good events contain all sample paths that \texttt{GLR-CUCB} reinitialize the UCB index of base arms after all change-points correctly within a small delay. On the other hand, the bad events contain all sample paths where either GLR change detector fails to detect the change in distribution or detects the change with a large delay. The cost incurred given the good events can be upper bounded by Lemma~\ref{lem:st_reg} and Lemma~\ref{col:exp_delay}. By upper bounding the probabilities of bad events via Lemma~\ref{lem:false_alarm} and Lemma~\ref{lem:cond_prob}, the cost incurred given the bad events is analyzable. Detailed proofs are presented in Appendix~\ref{sec:proof_regret_ub}.

\section{Regret Lower Bound}
\label{sec:reg_lb}
The lower bound for MAB problems has been studied extensively. Previously, the best available minimax lower bound for piecewise-stationary MAB was $\Omega(\sqrt{T})$ by~\citet{garivier2011upper}. Note that piecewise-stationary MAB is a special instance for piecewise-stationary CMAB in which every super arm is a base arm, thus this lower bound still holds for piecewise-stationary CMAB. We derive a tighter lower bound by characterizing the dependency on $N$ and $K$.
\begin{theorem}
\label{thm:minimax_lb}
If $K\geq 3$ and $T\geq M_1 N\frac{(K-1)^2}{K}$, then the worst-case regret for any policy is lower bounded by
\begin{align*}
    \mR(T)\geq M_2\sqrt{NKT},
\end{align*}
where $M_1=1/\log \tfrac{4}{3}$, $M_2=1/24\sqrt{\log \tfrac{4}{3}}$.
\end{theorem}
\begin{proof}
(Sketch) The high level idea is to construct a randomized `hard' instance which is appropriate to our setting~\citep{bubeck2012regret,besbes2014stochastic,lattimore2018bandit}, then analyze its regret lower bound which holds for any exploration policy. The construction of this  `hard' instance is as follows. 
 
We partition the time horizon into $N$ segments with equal length except for the last segment. In each segment, assume the rewards of all arms are Bernoulli distributions and stay unchanged. At each time step there is an optimal arm with expected reward of $\frac{1}{2}+\epsilon$ and the remaining arms have the same expected reward of $\frac{1}{2}$. The optimal arm will change in two consecutive segments by sampling uniformly at random from the remaining $K-1$ arms.

We then use Lemma A.1 in~\citet{auer2002nonstochastic} to upper bound the expected number of pulls to any arm being optimal under change of distributions, from the sub-optimal reward distribution to the optimal reward distribution (Bern($\frac{1}{2})$ to Bern($\frac{1}{2}+\epsilon$)). Given the upper bound of expected number of pulls to the optimal arm, we can lower bound the regret for any exploration policy. By properly tuning $\epsilon$ and after some additional steps, we can derive the minimax regret lower bound. The condition $K\geq 3$ comes from the fact that the lower bound needs to be non-trivial, and $T\geq M_1 N\frac{(K-1)^2}{K}$ comes from the tuning of $\epsilon$.

For the detailed proof, please refer to Appendix~\ref{sec:proof_regret_lb}.
\end{proof}
 The conditions for this minimax lower bound are mild, since in practice the number of base arms $K$ is usually much larger and we care about the long-term regret, in other words, large $T$ regime. 
 
 Our minimax lower bound shows that \texttt{GLR-CUCB} is nearly order-optimal with respect to all parameters. On the other hand, as a byproduct, this bound also indicates that \texttt{EXP3S}~\citep{auer2002nonstochastic} and \texttt{MUCB}~\citep{cao2019nearly} are nearly order-optimal for piecewise-stationary MAB, up to poly-logarithm factors. To be more specific, \texttt{EXP3S} and \texttt{MUCB} achieve regret $\mathcal{O}(\sqrt{NKT\log{KT}})$ and $\mathcal{O}(\sqrt{NKT\log{T}})$ respectively.

\section{Experiments}
\label{sec:exp}
We compare \texttt{GLR-CUCB} with five baselines from the literature, one variant of \texttt{GLR-CUCB}, and one oracle algorithm. Specifically, \texttt{DUCB}~\citep{kocsis2006discounted} and \texttt{MUCB}~\citep{cao2019nearly} are selected from piecewice-stationary MAB literature; \texttt{CUCB}~\citep{chen2013combinatorial}, \texttt{CTS}~\citep{wang2018thompson}, and \texttt{Hybrid}~\citep{zimmert2019beating} are selected from stochastic combinatorial semi-bandit literature. The variant of \texttt{GLR-CUCB}, termed \texttt{LR-GLR-CUCB}, uses different restart strategy. Instead of restarting the estimation of all bases arms once a change-point is detected, \texttt{LR-GLR-CUCB} uses local restart strategy (only restarts the estimation of the base arms that are detected to have changes in reward distributions). For the oracle algorithm, termed \texttt{Oracle-CUCB}, we assume the algorithm knows when the optimal super arm changes and restarts \texttt{CUCB} at these change-points. Note that this is stronger than knowing the change-points, since change in distribution does not imply change in optimal super arm. Experiments are conducted on both synthetic and real-world dataset for the $m$-set bandit problems, which aims to identify the $m$ arms with highest expected reward at each time step. Equivalently, the reward function $r_{\bm{\mu}_t}(S_t)$ is the summation of the expected rewards of $m$ base arms. Since \texttt{DUCB} and \texttt{MUCB} are originally designed for piecewise-stationary MAB, to adapt them to the piecewise-stationary CMAB setting, we treat every super arm as a single arm when we run these two algorithms. Reward distributions of base arms along time are postponed to Appendix~\ref{subsec:add_res}. The details about parameter tuning for all of these algorithms for different experiments are included in Appendix~\ref{subsec:imp_detail}.
\subsection{Synthetic Dataset}
In this case we design a synthetic piecewise-stationary combinatorial semi-bandit instance as follows:
\begin{itemize}
    \item Each base arm follows Bernoulli distribution.
    \item Only one base arm changes its distribution between two consecutive piecewise-stationary segments.
    \item Every piecewise-stationary segment is of equal length.
\end{itemize}
We let $T=5000$, $K=6$, $m=2$, and $N=5$. The average regret of all algorithms are summarized in Figure~\ref{fig:synthetic}.
\begin{figure}[!ht]
    \centering
    \includegraphics[width=0.7\columnwidth]{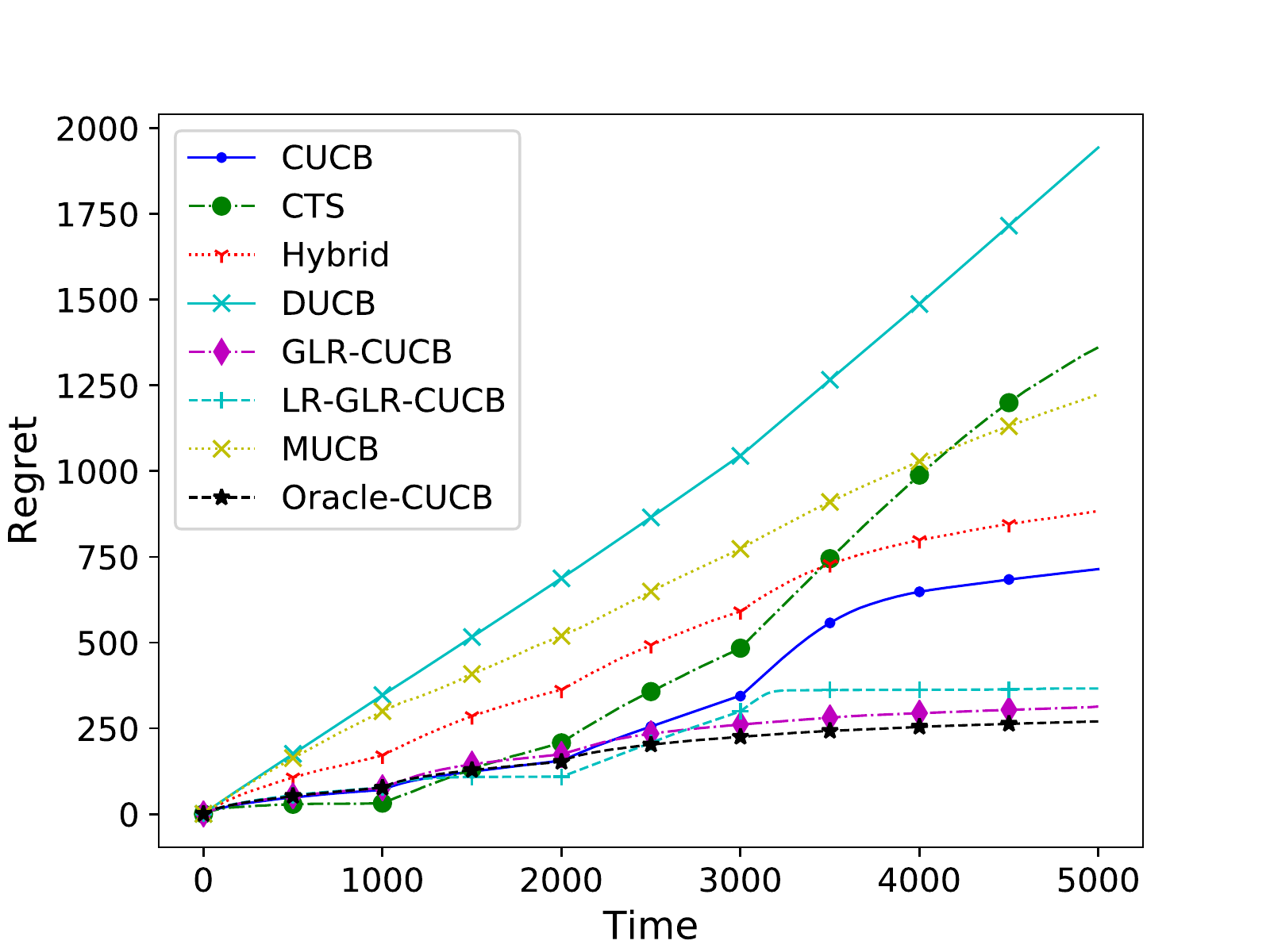}
    \caption{Expected cumulative regret for different algorithms on synthetic dataset.}
    \label{fig:synthetic}
\end{figure}
 Note that the optimal super arm does not change for the last three piecewise-stationary segments. Observe that the well-tuned GLR change-point detector is insensitive to change with small magnitude, which implicitly avoids unnecessary and costly global restart, since small change is less likely to affect the optimal super arm. Surprisingly, \texttt{GLR-CUCB} and \texttt{LR-GLR-CUCB} perform nearly as well as \texttt{Oracle-CUCB} and significantly better than other algorithms in regrets. In general, algorithms designed for stochastic CMAB outperform algorithms designed for piecewise-stationary MAB. The reason is when the horizon is small, the dimension of the action space dominates the regret, and this effect becomes more obvious when $m$ is larger. Although order-wise, the cost incurred by the change-point detection is much higher than the cost incurred by exploration.
 
 Note that our experiment on this synthetic dataset does not satisfy Assumption~\ref{assump:gap}. For example, the gap between the first segment and second segment is $0.6$, and we choose $\delta=\frac{20}{T}$ and $p=0.05\sqrt{(N-1)\log{T}/T}$ for $\texttt{GLR-CUCB}$, which means the length of the second segment should be at least 9874. However, the actual length of the second segment is only 1000. Thus our algorithm performs very well compared to other algorithms even if Assumption~\ref{assump:gap} is violated. If Assumption~\ref{assump:gap} is satisfied, \texttt{GLR-CUCB} can only perform better since it is easier to detect the change in distribution. 

\subsection{Yahoo! Dataset}
We adopt the benchmark dataset for the real-world evaluation of bandit algorithms from Yahoo!\footnote{Yahoo! Front Page Today Module User Click Log Dataset on https://webscope.sandbox.yahoo.com}. This dataset contains user click log for news articles displayed in the Featured Tab of the Today Module~\citep{li2011unbiased}. Every base arm corresponds to the click rate of one article. Upon arrival of a user, our goal is to maximize the expected number of clicked articles by presenting $m$ out of $K$ articles to the users.  
% where the reward function $r_{\bm{\mu}_t}(S_t)$ is the summation of click rate of $K$ articles. 

\subsubsection{\textbf{Yahoo! Experiment 1} ($K=6$, $m=2$, $N=9$).} We pre-process the dataset following~\citet{cao2019nearly}. To make the experiment nontrivial, we modify the dataset by: 1) the click rate of each base arm is enlarged by $10$ times; 2) Reducing the time horizon to $T=22500$. Results are in Figure~\ref{fig:yahoo1}.

\begin{figure}[!ht]
    \centering
    \includegraphics[width=0.7\columnwidth]{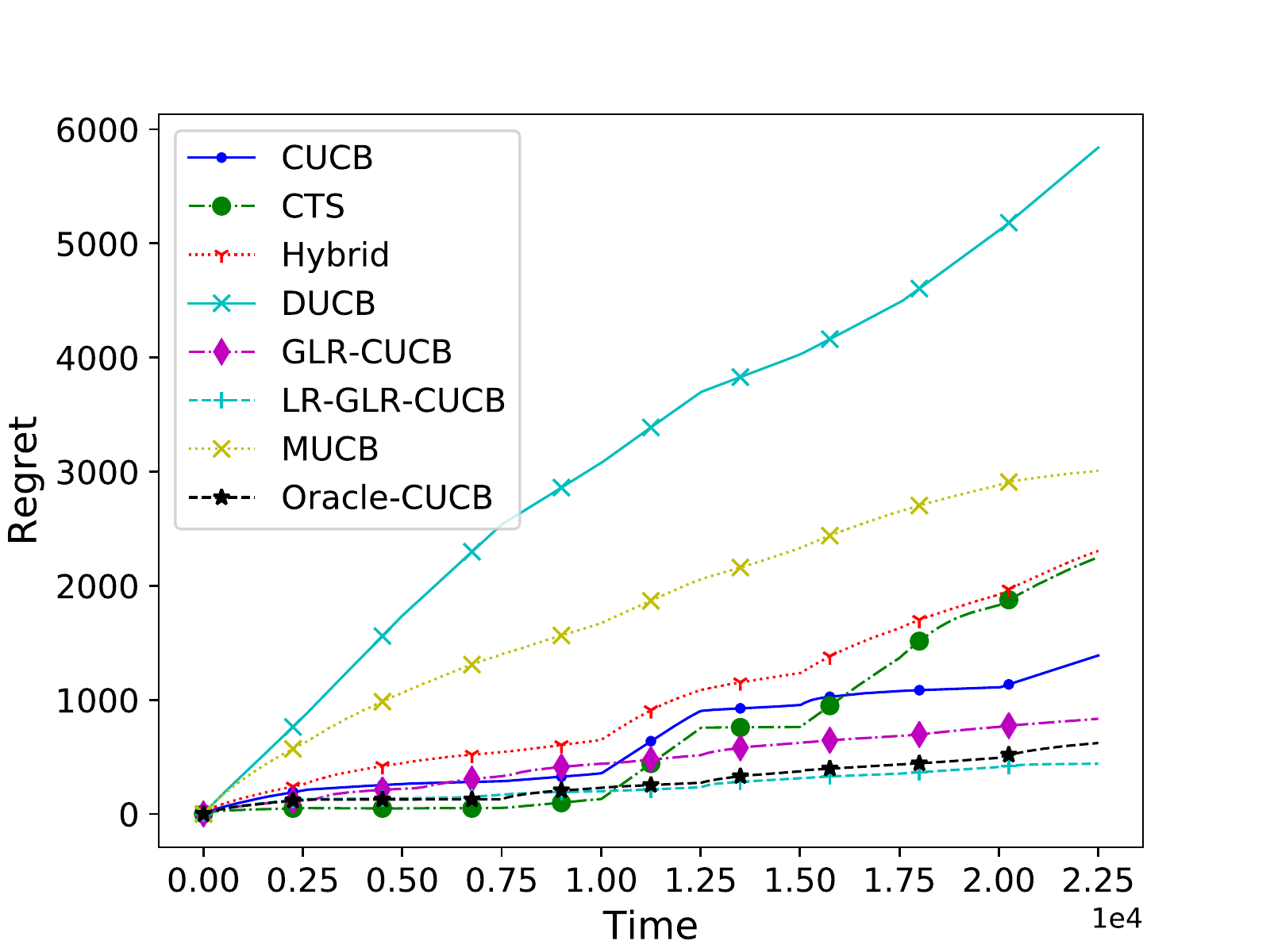}
    \caption{Expected cumulative regret for different algorithms on Yahoo! experiment 1.}
    \label{fig:yahoo1}
\end{figure}

Yahoo! experiment 1 is much harder than the synthetic problem, since it is much more non-stationary. Our experiments show \texttt{GLR-CUCB} still significantly outperforms other algorithms and only has a small gap with respect to \texttt{Oracle-CUCB}. Again, Assumption~\ref{assump:gap} does not hold for these two instances, thus we believe it is fair to compare \texttt{GLR-CUCB} with other algorithms. Unexpectedly, \texttt{LR-GLR-CUCB} performs even better than oracle-CUCB, which suggests there is still much to exploit in the piecewise-stationary bandits, since global restart has inferior performance in some cases, especially when the change in distribution is not significant. Additional experiments on Yahoo! dataset can be found in Appendix~\ref{subsec:add_res}.

\section{Conclusion and Future Work}
\label{sec:conclusion}
We have developed the first efficient and general algorithm for piecewise-stationary CMAB, termed \texttt{GLR-UCB}, which extends \texttt{CUCB}~\citep{chen2013combinatorial}, by incorporating a GLR change-point detector. We analyze the regret upper bound of \texttt{GLR-CUCB} on the order of $\mathcal{O}(\sqrt{NKT\log T})$, and prove the minimax lower bound for piecewise-stationary MAB and CMAB on the order of $\Omega(\sqrt{NKT})$, which shows our algorithm is nearly order-optimal within poly-logarithm factors. Experimental results show our proposed algorithm outperforms other state-of-the art algorithms. 

Future work includes designing algorithms for piecewise-stationary CMAB with better restart strategy. \texttt{GLR-CUCB} restarts whenever the GLR change-point detector declares the reward distribution of one base arm changes, but this restart is very likely unnecessary, because change-point with small magnitude might not change the optimal superarm. Another very challenging unsolved problem is whether one can close the gap between the regret upper bound and the minimax regret lower bound. Specifically, develop algorithm which is order-optimal for piecewise-stationary CMAB. It is also important to study other non-stationary settings for combinatorial semi-bandits, such as the dynamic case (with bounded total variation throughout time)~\citep{chen2020nonstationary}.
% \section*{Acknowledgement}
% \input{./acknowledgement.tex}

\bibliographystyle{plainnat}
\bibliography{refs.bib}

\label{sec:apdx}
\newpage
\onecolumn
\appendix
\begin{center}
\textbf{\Large Appendices}
\end{center}

\section{Detailed Proofs of Theorem~\ref{thm:regret_ub}}
\label{sec:proof_regret_ub}
\subsection{Proof of Lemma~\ref{lem:st_reg}}
\begin{proof}
Define $N_k(t)$ as a counter for each base arm $k\in\mathcal{K}$ at time $t$ (note that it is different from the counter $n_k$ defined in the algorithm) and update the counters in each round as follows: (1) After the $K$ initialization rounds, set $N_k(K)=1, \forall k\in\mathcal{K}$. (2) For a round $t>K$, if $S_t$ is bad, then increase $N_{k^*}(t)$ by one, where $k^*=\arg\min_{k\in\mathcal{K}}N_{k}(t-1)$.
By definition,  the total number of bad rounds at time $T$ is no more than $\sum_{k\in\mathcal{K}}N_k(T)$. Throughout  the proof, we will use $n_k(t)$ to denote the total number of times arm $k$ is played by the agent till time $t$, to emphasize the time dependency in order to make the proof more readable.\\
Note that when a bad super arm is played, it incurs loss at most $\Delta^{\text{max}, 1}_{\textup{opt}}$. Thus,
\begin{align*}
     \mathcal{R}(T)&\leq \Delta^{\text{max},1}_{\textup{opt}}\mathbb{E}\left[\sum_{k\in\mathcal{K}}N_k(T)\right].
\end{align*}
Thus it suffices to upper bound $\mathbb{E}\left[\sum_{k\in\mathcal{K}}N_k(T)\right]$ to upper bound the cumulative regret. Let $l_t=\frac{6KL^2\log t }{\left(\Delta^{\text{min}, 1}_{\textup{opt}}\right)^2}$, we have
\begin{align*}
    & \sum_{k\in\mathcal{K}}N_k(T) - Kl_T=\sum_{k\in\mathcal{K}}N_k(T)\indfunc{\{\tau_1\leq T\}} + \sum_{k\in\mathcal{K}}N_k(T)\indfunc{\{\tau_1\geq T\}} - K(l_T+1)\\
    &=\sum_{k\in\mathcal{K}}N_k(T)\indfunc\left\{\tau_1\leq T\right\}+ \sum_{t=K+1}^{T}\indfunc\left\{S_t\in\mathcal{S}_B^1, \tau_1\geq T\right\} - Kl_T\\
    &= \sum_{k\in\mathcal{K}}N_k(T)\indfunc\left\{\tau_1\leq T\right\}+\sum_{t=K+1}^{T}\indfunc\left\{S_t\in\mathcal{S}_B^1, \tau_1\geq T, t\;\mathrm{mod}\; \left\lfloor \frac{K}{p}\right\rfloor \in \mK\right\}+\\
    &\quad\sum_{t=K+1}^{T}\indfunc\left\{S_t\in\mathcal{S}_B^1, \tau_1\geq T,  t\;\mathrm{mod}\; \left\lfloor \frac{K}{p}\right\rfloor \not\in \mK\right\} - Kl_T\\
    &\leq T\indfunc\left\{\tau_1\leq T\right\}+ \sum_{t=K+1}^{T}\indfunc\left\{ t\;\mathrm{mod}\; \left\lfloor \frac{K}{p}\right\rfloor \in \mK\right\} + \sum_{t=K+1}^{T}\sum_{k\in\mathcal{K}}\indfunc\left\{S_t\in\mathcal{S}_B^1, N_k(t)>N_k(t-1), N_k(t-1)>l_T\right\}\\
   &\leq T\indfunc\left\{\tau_1\leq T\right\}+\sum_{t=K+1}^{T}\indfunc\left\{ t\;\mathrm{mod}\; \left\lfloor \frac{K}{p}\right\rfloor \in \mK\right\} + \sum_{t=K+1}^{T}\sum_{k\in\mathcal{K}}\indfunc\left\{S_t\in\mathcal{S}_B^1, N_k(t)>N_k(t-1), N_k(t-1)>l_t\right\}\\
   &=  T\indfunc\left\{\tau_1\leq T\right\}+\sum_{t=K+1}^{T} \indfunc\left\{ t\;\mathrm{mod}\; \left\lfloor \frac{K}{p}\right\rfloor \in \mK\right\} + \sum_{t=K+1}^{T}\indfunc\left\{S_t\in\mathcal{S}_B^1, \forall k\in S_t, N_k(t-1)>l_t\right\}\\
   &\leq T\indfunc\left\{\tau_1\leq T\right\}+\sum_{t=K+1}^{T} \indfunc\left\{ t\;\mathrm{mod}\; \left\lfloor \frac{K}{p}\right\rfloor \in \mK\right\} + \sum_{t=K+1}^{T}\indfunc\left\{S_t\in\mathcal{S}_B^1, \forall k\in S_t, n_k(t-1)>l_t\right\}.
\end{align*}
The next step is to show $\mathbb{P}\left(S_t\in\mathcal{S}_B^1, \forall k\in S_t, n_k(t-1)>l_t\right)\leq \frac{2K}{t^2},\forall t\in[T]$. Let $n_k(t)$ be the number of times arm $k$ is played in the first $t$ rounds, $\hat{\mu}^1_{k, m}$ be the empirical mean of $m$ samples of $k$th arm at the first piecewise-stationary segment, $\text{UCB}_k(t)$ be the UCB index of $k$th arm at time $t$, and $\mu_k^1$ be the actual mean of $k$th arm during the first piecewise-stationary segment. For any $k\in\mathcal{K}$, by applying the Hoeffding's inequality, we have
\begin{align*}
    \mathbb{P}\left(\hat{\mu}_{k, n_k(t-1)}^1-\mu_k^1\geq\sqrt{\frac{3\log {t}}{2n_k(t-1)}}\right)\leq \sum_{m=1}^{t-1}\mathbb{P}\left(\hat{\mu}^1_{k, 1:m}-\mu_k^1\geq\sqrt{\frac{3\log {t}}{2m}}\right)\leq t e^{-3\log t}=t^{-2}.
\end{align*}
Define the event $E_t=\left\{\forall k\in\mathcal{K}, \hat{\mu}^1_{k, n_k(t-1)}-\mu_k^1\leq\sqrt{\frac{3\log {t}}{2n_k(t-1)}}\right\}$. By the union bound we have $\mathbb{P}\left(\overline{E}_t\right)\leq Kt^{-2}$. However, we can show that $E_t\cap\left\{S_t\in\mathcal{S}_B^1, \forall k\in S_t, n_k(t-1)>l_t\right\}=\emptyset$. In other words, these two events are mutually exclusive. The reason is if both of these events hold, we have
\begin{align*}
    r_{\pmb{\mu}^1}(S_t) + L\sqrt{K\frac{6\log t}{l_t}}\overset{(a)}{>} r_{\pmb{\mu}^1}(S_t)+L\sqrt{K\max_{k\in S_t}\frac{6\log t}{n_k(t-1)}} \overset{(b)}{\geq} r_{\text{UCB}}(S_t)\overset{(c)}{\geq} \alpha\max_{\tilde{S}\in\mathcal{F}}r_{\text{UCB}}(\tilde{S})\overset{(d)}{\geq} \alpha\max_{\tilde{S}\in\mathcal{F}}r_{\pmb{\mu}^1}(\tilde{S}),
\end{align*}
where $\pmb{\mu}^1$ is the mean vector of the base arms at the first piecewise-stationary segment. As for these above inequalities, (a) holds since $n_k(t-1)>l_t, \forall k\in S_t$; (b) holds by the $L$-Lipschitz property of $r_{\mu}(\cdot)$ and $\left\lVert\text{UCB}-\pmb{\mu}^1\right\rVert\leq \sqrt{K\max_{k\in S_t}\frac{6\log t}{n_k(t-1)}}$. Note that the upper bound of the norm of vector difference comes from the fact that $|\hat{\mu}_{k, n_k(t-1)}^1-\mu_k^1|\leq\sqrt{\frac{3\log t}{2n_k(t-1)}}$, and $\text{UCB}_k(t-1) = \hat{\mu}^1_{k, n_k(t-1)} + \sqrt{\frac{3\log {t}}{2n_k(t-1)}}$; (c) holds by the definition of $\alpha$-approximation oracle; (d) holds by the monotone property of $r_{\mu}(\cdot)$ and $\text{UCB}\geq\pmb{\mu}^1$. However this contradicts the fact that $ r_{\pmb{\mu}^1}(S_t) + L\sqrt{K\frac{6\log t}{l_t}}\leq \alpha\max_{\tilde{S}\in\mathcal{F}}r_{\pmb{\mu}^1}(\tilde{S})$, since $L\sqrt{K\frac{6\log t}{l_t}}=\Delta^{\text{min},1}_{\textup{opt}}$, which implies that these two events are mutually exclusive.
Thus, 
\begin{align*}
\mathbb{P}\left(S_t\in\mathcal{S}_B, \forall k\in S_t, n_k(t-1)>l_t\right)\leq \mathbb{P}\left(\overline{E}_t\right)\leq 2Kt^{-2}.
\end{align*}
To sum up, 
\begin{align*}
    \mathbb{E}\left[\sum_{k\in\mathcal{K}}N_k(T)\right] &\leq K(l_T+1) + p(T-K) + \sum_{t=K+1}^{T}\mathbb{P}\left(S_t\in\mathcal{S}_B^1, \forall k\in S_t, n_k(t-1)>l_t\right)+T\mathbb{P}(\tau_1\leq T)\\
    &\leq K(l_T+1) + pT + \sum_{t=1}^{T}\frac{K}{t^2} + T\mathbb{P}(\tau_1\leq T)\leq \frac{6L^2K^2\log T}{(\Delta^{\text{min}, 1}_{\textup{opt}})^2} + pT + \left(\frac{\pi^2}{6}+1\right)K + T\mathbb{P}(\tau_1\leq T).
\end{align*}
The proof is done. 
\end{proof}

\subsection{Proof of Lemma~\ref{lem:false_alarm}}
\begin{proof}
Define $\tau_{k,1}$ as the first change-point detection time of the $k$th base arm, and then $\tau_1=\min_{k\in\mK}\tau_{k,1}$ as $\texttt{GLR_CUCB}$ restarts the whole algorithm if change-point is detected on any of the base arms. Applying the union bound to the false alarm probability $\Prob\left(\tau_1\le T\right)$, we have that
\begin{equation*}
    \Prob\left(\tau_1\le T\right)\le\sum_{k=1}^K\Prob\left(\tau_{k,1}\le\tau\right).
\end{equation*}
Recall the GLR statistic defined in Eq.~\eqref{eq:GLR_statistic}, and substitute it into $\Prob\left(\tau_{k,1}\le T\right)$, 
\begin{align*}
    \Prob\left(\tau_{k,1}\le T\right)&\le\Prob\left[\exists n\le n_k~\mbox{and}~ s<n:s\times\text{kl}\left(\hat{\mu}^1_{k,1:s},\hat{\mu}^1_{k,1:n}\right)+(n-s)\times\text{kl}\left(\hat{\mu}^1_{k,s+1:n},\hat{\mu}^1_{k,1:n}\right)>\beta(n,\delta)\right]\\
    &\overset{(a)}{\le}\Prob\left[\exists n\le T~\mbox{and}~ s<n:s\times\text{kl}\left(\hat{\mu}^1_{k,1:s},\hat{\mu}^1_{k,1:n}\right)+(n-s)\times\text{kl}\left(\hat{\mu}^1_{k,s+1:n},\hat{\mu}^1_{k,1:n}\right)>\beta(n,\delta)\right]\\
    &\overset{(b)}{\le}\Prob\left[\exists n\le T~\mbox{and}~ s<n:s\times\text{kl}\left(\hat{\mu}^1_{k,1:s},\mu_k^1\right)+(n-s)\times\text{kl}\left(\hat{\mu}^1_{k,s+1:n},\mu^1_k\right)>\beta(n,\delta)\right]\\
    &\overset{(c)}{\le}\sum_{s=1}^T\Prob\left[\exists s<n:s\times\text{kl}\left(\hat{\mu}^1_{k,1:s},\mu_k^1\right)+(n-s)\times\text{kl}\left(\hat{\mu}^1_{k,s+1:n},\mu^1_k\right)>\beta(n,\delta)\right]\\
    &\le\sum_{s=1}^T\Prob\left[\exists r\in\N:s\times\text{kl}\left(\hat{\mu}^1_{k,1:s},\mu_k^1\right)+r\times\text{kl}\left(\hat{\mu}^1_{k,s+1:s+r},\mu^1_k\right)>\beta(s+r,\delta)\right]\\
    &\overset{(d)}{\le}\sum_{s=1}^T\frac{\delta}{3s^{3/2}}
    \overset{(e)}{\le}\sum_{s=1}^\infty\frac{\delta}{3s^{3/2}}
    \le\delta,
\end{align*}
where $\hat{\mu}_{s,s'}^1$ is the mean of the rewards generated from the distribution $f_k^1$ with expected reward $\mu_k^1$ from time step $s$ to $s'$ as only stationary scenario is considered here. Here, inequality (a) is because of the fact that $n_k\le T$; inequality (b) holds since:
\begin{equation*}
    s\times\text{kl}\left(\hat{\mu}_{1:s},\hat{\mu}_{1:n}\right)+(n-s)\times\text{kl}\left(\hat{\mu}_{s+1:n},\hat{\mu}_{1:n}\right)=\inf_{\lambda\in[0,1]}\left[s\times\text{kl}\left(\hat{\mu}_{1:s},\lambda\right)+(n-s)\times\text{kl}\left(\hat{\mu}_{s+1:n},\lambda\right)\right];
\end{equation*}
inequality (c) is because of the union bound; inequality (d) is according to the Lemma 10 in~\citet{besson2019generalized}; inequality (e) holds due to the Riemann zeta function $\zeta(s)$, when $s=1.5$, $\zeta(s)<2.7$. Summing over $k\in\mK$, we conclude that $\Prob\left(\tau_1\le T\right)\le K\delta$.
\end{proof}

\subsection{Proof of Corollary~\ref{col:exp_delay}}
\begin{proof}
 By the definition of $\mC^{(i)}$, it follows directly that the conditional expected detection delay is upper bounded by $d_i$.
\end{proof}
\subsection{Proof of Theorem~\ref{thm:regret_ub}}
\begin{proof}
Define the good event $F_i=\{\tau_i>\nu_i\}$ and good event $D_i=\{\tau_i\le\nu_i+d_i\}$, $\forall 1\le i\le N-1$. Recall the definition of the good event $\mC^{(i)}$ that all the change-points up to $i$th one have been detected successfully and efficiently in Eq.~\eqref{eq:C_event}, and we can find that $\mC^{(i)}=F_1\cap D_1\cap\cdots \cap F_i\cap D_i$ is the intersection of the event sequence of $F_i$ and $D_i$ up to the $i$th change-point. By first decomposing the expected $\alpha$-approximation cumulative regret with respect to the event $F_1$, we have that
\begin{align*}
    \mR(T)=\E\left[R(T)\right]&=\E\left[R(T)\indfunc{\{F_1\}}\right]+\E\left[R(T)\indfunc{\{\overline{F}_1\}}\right]\\
    &\le\E\left[R(T)\indfunc{\{F_1\}}\right]+T\Delta^{\text{max}}_{\text{opt}}\Prob(\overline{F}_1)\\
    &\le\E\left[R(\nu_1)\indfunc{\{F_1\}}\right]+\E\left[R(T-\nu_1)\right]+T\Delta^{\text{max}}_{\text{opt}}K\delta\\
    &\overset{(a)}{\le}\widetilde{C}_1+\Delta^{\text{max},1}_{\text{opt}}\nu_1 p+\E\left[R(T-\nu_1)\right]+T\Delta^{\text{max}}_{\text{opt}}K\delta,
\end{align*}
where $R(t') = \alpha\sum_{t=1}^{t'}\max_{S\in\mF}r_{\bm{\mu}_t}(S)-\sum_{t=1}^{t'}r_{\bm{\mu}_t}(S_{t})$ and $R(t''-t') = \alpha\sum_{t=t'+1}^{t''}\max_{S\in\mF}r_{\bm{\mu}_t}(S)-\sum_{t=t'+1}^{t''}r_{\bm{\mu}_t}(S_{t})$. Here, $\Prob(\overline{F}_1)$ enters the equation by Lemma~\ref{lem:false_alarm}, and inequality $(a)$ is due to the bound in the Lemma~\ref{lem:st_reg}.

Next, we turn to bounding $\E\left[R(T-\nu_1)\right]$. By the law of total expectation, we have
\begin{align*}
    \E\left[R(T-\nu_1)\right]&\le\E\left[R(T-\nu_1)\bigm|F_1\cap D_1\right]+T\Delta^{\text{max}}_{\text{opt}}(1-\Prob(F_1\cap D_1))\\
    & = \E\left[R(T-\nu_1)\bigm|F_1\cap D_1\right]+T\Delta^{\text{max}}_{\text{opt}}\Prob(\overline{F}_1\cup \overline{D}_1)\\
    &\le\E\left[R(T-\nu_1)\bigm|F_1\cap D_1\right]+T\Delta^{\text{max}}_{\text{opt}}(K+1)\delta,
\end{align*}
where the probability $\Prob(\overline{F}_1\cup \overline{D}_1)$ is from applying the union bound to Lemma~\ref{lem:cond_prob}. Furthermore, notice that
\begin{align*}
    \E\left[R(T-\nu_1)\bigm|F_1\cap D_1\right]&=\E\left[R(T-\nu_1)\bigm|\mC^{(1)}\right]\\
    &\le\E\left[R(T-\tau_1)\bigm|\mC^{(1)}\right]+E\left[R(\tau_1-\nu_1)\bigm|\mC^{(1)}\right]\\
    &\le\E\left[R(T-\nu_1)\bigm|\mC^{(1)}\right]+\Delta^{\text{max},2}_{\text{opt}}d_1,
\end{align*}
where $\E\left[R(\tau_1-\nu_1)\bigm|\mC^{(1)}\right]$ is upper-bounded according to Corollary~\ref{col:exp_delay}. 
By combining the previous steps, we have that 
\begin{align*}
  \mR(T)\le\E\left[R(T-\nu_1)\bigm|\mC^{(1)}\right]+\widetilde{C}_1+\Delta^{\text{max},1}_{\text{opt}}\nu_1p+\Delta^{\text{max},2}_{\text{opt}}d_1+3T\Delta^{\text{max}}_{\text{opt}}K
\delta.  
\end{align*}

Similarly,
\begin{align*}
  \E\left[R(T-\nu_1)\bigm|\mC^{(1)}\right]&\le\E\left[R(T-\nu_1)\indfunc{\{F_2\}}\bigm|\mC^{(1)}\right]+T\Delta^{\text{max}}_{\text{opt}}\Prob\left(\overline{F}_2\bigm|\mC^{(1)}\right)\\
  &\le\E\left[R(\nu_2-\nu_1)\indfunc{\{F_2\}}\bigm|C^{(1)}\right]+\E\left[R(T-\nu_2)\bigm|\mC^{(1)}\right]+T\Delta^{\text{max}}_{\text{opt}}K\delta\\
  &\le \widetilde{C}_2+\Delta^{\text{max},2}_{\text{opt}}(\nu_2-\nu_1)p+\E\left[R(T-\nu_2)\bigm|\mC^{(1)}\right]+T\Delta^{\text{max}}_{\text{opt}}K\delta,
\end{align*}
where $\Prob\left(\overline{F}_2\bigm|\mC^{(1)}\right)$ directly follows Lemma~\ref{lem:cond_prob}. For the term $\E\left[R(T-\nu_2)\bigm|\mC^{(1)}\right]$, we further decompose is as follows,
\begin{align*}
\E\left[R(T-\nu_2)\bigm|\mC^{(1)}\right]&\le\E\left[R(T-\nu_2)\bigm|F_2\cap D_2\cap \mC^{(1)}\right]+T\Delta^{\text{max}}_{\text{opt}}\left(1-\Prob\left(F_2\cap D_2\bigm|\mC^{(1)}\right)\right)\\
& = \E\left[R(T-\nu_2)\bigm|F_2\cap D_2\cap \mC^{(1)}\right]+T\Delta^{\text{max}}_{\text{opt}}\Prob\left(\overline{F}_2\cup \overline{D}_2\bigm|\mC^{(1)}\right)\\
&\le E\left[R(T-\nu_2)\bigm|\mC^{(2)}\right]+T\Delta^{\text{max}}_{\text{opt}}(K+1)\delta,
\end{align*}
where $\Prob\left(\overline{F}_2\cup\overline{D}_2\bigm|\mC^{(1)}\right)$ follows Lemma~\ref{lem:cond_prob} by applying the union bound.
For $\E[R(T-\nu_2)|\mC^{(2)}]$, we have
\begin{align*}
    \E\left[R(T-\nu_2)\bigm|\mC^{(2)}\right]&\le\E\left[R(T-\tau_2)|\mC^{(2)}\right]+\E\left[R(\tau_2-\nu_2)\bigm|\mC^{(2)}\right]\\
    &\le\E\left[R(T-\nu_2)\bigm|\mC^{(2)}\right]+\Delta^{\text{max},3}_{\text{opt}}\E\left[\tau_2-\nu_2\bigm|\mC^{(2)}\right]\\
    & \le\E[R(T-\nu_2)|\mC^{(2)}]+\Delta^{\text{max},3}_{\text{opt}}d_2.
\end{align*}
Wrapping up previous steps, we have that 
\begin{align*}
  \mR(T)\le\E\left[R(T-\nu_2)\bigm|\mC^{(2)}\right]+\widetilde{C}_1+\widetilde{C}_2+\Delta^{\text{max},1}_{\text{opt}}\nu_1p+\Delta^{\text{max},2}_{\text{opt}}(\nu_2-\nu_1)p+\Delta^{\text{max},2}_{\text{opt}}d_1+\Delta^{\text{max},3}_{\text{opt}}d_2+6T\Delta^{\text{max}}_{\text{opt}}K\delta.  
\end{align*}
 Recursively, we can bound $E\left[R(T-\nu_i)\bigm|\mC^{(i)}\right]$ by applying the same method as before, and the upper bound on the regret of $\texttt{GLR-CUCB}$ is
\begin{equation*}
    \mR(T)\le\sum_{i=1}^N \widetilde{C}_i+\Delta^{\text{max}}_{\text{opt}}Tp+\sum_{i=1}^{N-1}\Delta^{\text{max},i+1}_{\text{opt}}d_i+3NT\Delta^{\text{max}}_{\text{opt}}K\delta.
\end{equation*}
\end{proof}

\subsection{Proof of Corollary~\ref{col:regret_tuning}}
\begin{proof}
Notice that
\begin{align*}
    d_i&=\left\lceil\frac{4K}{p(\Delta_{\textup{change}}^i)^2}\beta(T,\delta)+\frac{K}{p}\right\rceil\le\frac{4K}{p(\Delta_{\text{change}}^i)^2\beta(T,\delta)}+\frac{2K}{p}\\
    &\overset{(a)}{\le}\frac{4K}{p(\Delta_{\text{change}}^i)^2}\left[\log(\frac{3T^{3/2}}{\delta})+8\log\left(1+\frac{\log(\frac{3T^{3/2}}{\delta})}{2}+\sqrt{\log(\frac{3T^{3/2}}{\delta}})\right)+6\log(1+\log{T})\right]+\frac{2K}{p}\\
    &\overset{(b)}{\le}\frac{\frac{20K\log{T}+o(K\log{T})}{(\Delta_{\text{change}}^i)^2}+2K}{p}\lesssim \frac{K\log{T}}{p(\Delta_{\text{change}}^{\text{min}})^2},
\end{align*}
where inequalities $(a)$ and $(b)$ hold when $\log(3T^{5/2})\ge 10$ (equals to $T\ge 36$). 
By plugging $d_i$ into Theorem~\ref{thm:regret_ub}, we have,
\begin{equation*}
    \mR(T)\lesssim\frac{NK^2\log{T}\Delta^{\text{max}}_{\textup{opt}}}{\left(\Delta^{\text{min}}_{\textup{opt}}\right)^2}+\Delta^{\text{max}}_{\text{opt}}Tp+\frac{N\Delta^{\text{max}}_{\text{opt}}K\log{T}}{p(\Delta_{\text{change}}^{\text{min}})^2}+3NK\Delta^{\text{max}}_{\text{opt}}.
\end{equation*}
Combining the above analysis we conclude the corollary.
\end{proof}

\section{Proof of Theorem~\ref{thm:minimax_lb}}
\label{sec:proof_regret_lb}
\begin{proof}
Consider a randomized hard instance defined as follows. First partition the decision horizon into $N$ segments, where the first $N-1$ segments are of size $l=\lceil \frac{T}{N}\rceil$, and the last segment's length is $T-(N-1)l$. Denote these $N$ segments as $\mathcal{B}_1, \mathcal{B}_2, ..., \mathcal{B}_N$. In each segment, assume the rewards of all arms follow Bernoulli distribution. Furthermore, we assume the following
\begin{itemize}
    \item $\mu_k^{t_1}=\mu_k^{t_2}$, $\forall k\in\mathcal{K}$ and $t_1, t_2\in \mathcal{B}_i$ for some $i\in[N]$ (Here, with a little inconsistency with the previous notation, $\mu_k^t$ is the mean of the reward of arm $k$ at time $t$). 
    \item $\mu_k^t\in\{\frac{1}{2}, \frac{1}{2}+\epsilon\}$, $\forall k\in\mathcal{K}$ and $\forall t\in\mathcal{T}$, where $\epsilon$ is a parameter that needs to be tuned later. 
    \item $\sum_{i=1}^{K}\mu_i^t=\frac{K}{2}+\epsilon$, $\forall t\in\mathcal{T}$.
\end{itemize}
Basically at each time instant there is an optimal arm with expected reward $\frac{1}{2}+\epsilon$, and the rest of the base arms have the same reward $\frac{1}{2}$. In addition, in each segment, the reward distribution of all arms stay unchanged. Lastly, we need to specify how the reward distributions change in the consecutive blocks. In $B_i$ if the optimal arm is $k_i^*$, then the optimal arm in $B_{i+1}$ will be drawn uniformly at random from $\mathcal{K}\setminus k_i^*$. 

Now we finish setting up the randomized piecewise-stationary bandit instance. To prove the theorem, we need some more notations. Fix a policy $\pi$, and fix a segment $i\in[N]$, let $k_i^*$ denote the best arm in that segment. We denote by $P_{k_i^*}^i$ the probability distribution conditioned on arm $k_i^*$ being the best arm in segment $i$, and by $P_0$ the probability distribution with respect to non-optimal rewards for each arm, i.e. with expected reward $\frac{1}{2}$. Furthermore, we denote by $\E_{k_i^*}^i[\cdot]$ and $\E_0[\cdot]$ the respective expectations. Lastly, $N_k^i$ is the number of times arm $k$ was selected in segment $i$ and $X\in\{0, 1\}^T$ is the binary reward sequence. To prove the theorem, we need Lemma A.1 in~\citet{auer2002nonstochastic}, which we state for completeness. 
\begin{lemma}
(Lemma A.1 in~\citet{auer2002nonstochastic}) Let $f:\{0,1\}^T\rightarrow[0, M]$ be any bounded function defined on reward sequence $X$, then $\forall k\in\mathcal{K}$,
\begin{align*}
    \E_k^i[f(X)]\leq \E_0[f(X)]+\frac{T}{2}\sqrt{D_{KL}(P_0||P_k^i)\E_0[N_{k_i^*}^i]}.
\end{align*}
\end{lemma}
Applying the above lemma with $f(X)=N_{k_i^*}^i$, we have
\begin{align*}
    \E_{k_i^*}^i[N_{k_i^*}^i]\leq \E_0[N_{k_i^*}^i]+\frac{|\mathcal{B}_i|}{2}\sqrt{\log{\frac{1}{1-4\epsilon^2}}\E_0[N_{k_i^*}^i]}.
\end{align*}
Summing over all arms, we have
\begin{align}
    \sum_{k_i^*=1}^{K}\E_{k_i^*}^i[N_{k_i^*}^i]&\leq \sum_{k_i^*=1}^{K}\E_0[N_{k_i^*}^i] + \sum_{k_i^*=1}^{K}\frac{|\mathcal{B}_i|}{2}\sqrt{\log{\frac{1}{1-4\epsilon^2}}\E_0[N_{k_i^*}^i]}\notag\\
    &\leq |\mathcal{B}_i|+\frac{|\mathcal{B}_i|}{2}\sqrt{|\mathcal{B}_i|K\log{\frac{1}{1-4\epsilon^2}}},\label{num_ub}
\end{align}
where the last equality holds since $\sum_{k_i^*=1}^{K}\E_0[N_{k_i^*}^i]=|\mathcal{B}_i|$ and Jensen's inequality. Then we can lower bound the regret for any policy $\pi$ for this randomized hard instance. Let $\mathcal{R}^{\pi}_T$ be the regret of policy $\pi$ with horizon $T$, we have,
\begin{align*}
    \mathcal{R}^{\pi}(T)&=(\tfrac{1}{2}+\epsilon)T-\E_{\pi}[\sum_{t=1}^{T}\mu_t]\\
                        &\geq (\tfrac{1}{2}+\epsilon)T-(\tfrac{1}{2}+\epsilon)\tfrac{1}{K-1}\sum_{i=1}^{N}\sum_{k_i^*=1}^{K}\E_{k_i^*}^i[N_{k_i^*}^i]-\tfrac{1}{2}(T-\tfrac{1}{K-1}\sum_{i=1}^{N}\sum_{k_i^*=1}^{K}\E_{k_i^*}^i[N_{k_i^*}^i])\\
                        &=\epsilon(T-\tfrac{1}{K-1}\sum_{i=1}^{N}\sum_{k_i^*=1}^{K}\E_{k_i^*}^i[N_{k_i^*}^i])\\
                        &\overset{(a)}{\geq}\epsilon(T-\tfrac{1}{K-1}(\sum_{i=1}^{N}|\mathcal{B}_i|+\frac{|\mathcal{B}_i|}{2}\sqrt{K|\mathcal{B}_i|\log{\tfrac{1}{1-4\epsilon^2}}}))\\
                        &=\epsilon T - \tfrac{\epsilon T}{K-1} -\tfrac{\epsilon T}{2(K-1)}\sqrt{\frac{KT}{N}\log{\tfrac{1}{1-4\epsilon^2}}}\\
                        &\overset{(b)}{\geq}\tfrac{\epsilon T}{2}-\tfrac{\epsilon^2 T}{K-1}\sqrt{\tfrac{KT}{N}\log{\tfrac{4}{3}}}.
\end{align*}
where $(a)$ is due to inequality~(\ref{num_ub}), and $(b)$ holds by $K\geq 3$, $4\epsilon^2\leq\frac{1}{4}$ and $\log{\frac{1}{1-x}}\leq 4\log(\frac{4}{3})x$ for all $x\in[0, \frac{1}{4}]$. We finish the proof by setting $\epsilon=\frac{K-1}{4\sqrt{TK\log{(\frac{4}{3})}}}$.
\end{proof}
\section{Supplementary Materials for Experiments}
\label{sec:sup_exp}
\subsection{Additional Experimental Results}
\label{subsec:add_res}
Reward distributions of base arms along time for three experiments are summarized in Figure~\ref{fig:reward_dist}.
\begin{figure*}[htbp]
    \centering
    \includegraphics[width=0.3\textwidth]{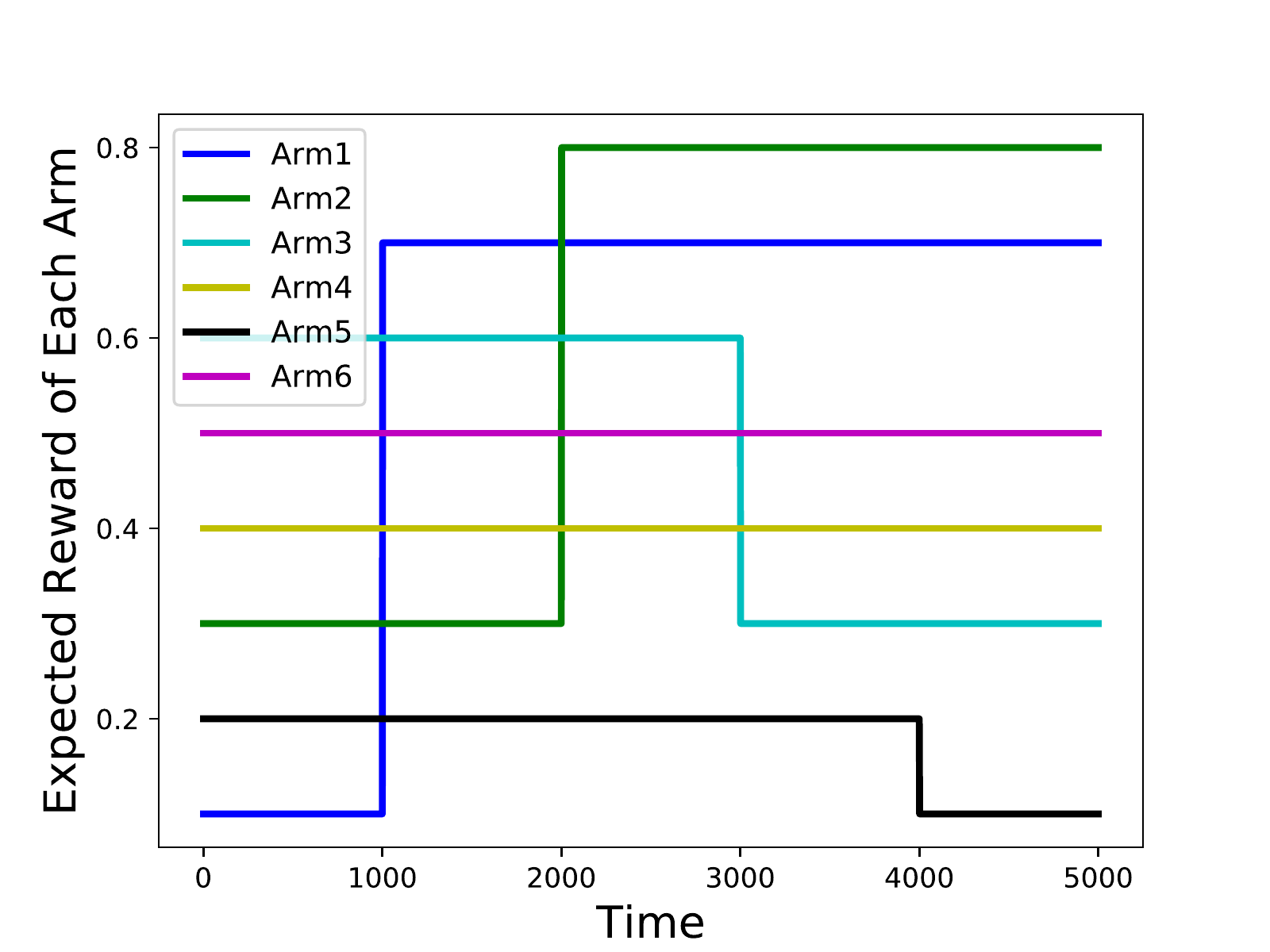}\
    \includegraphics[clip, width=0.3\textwidth]{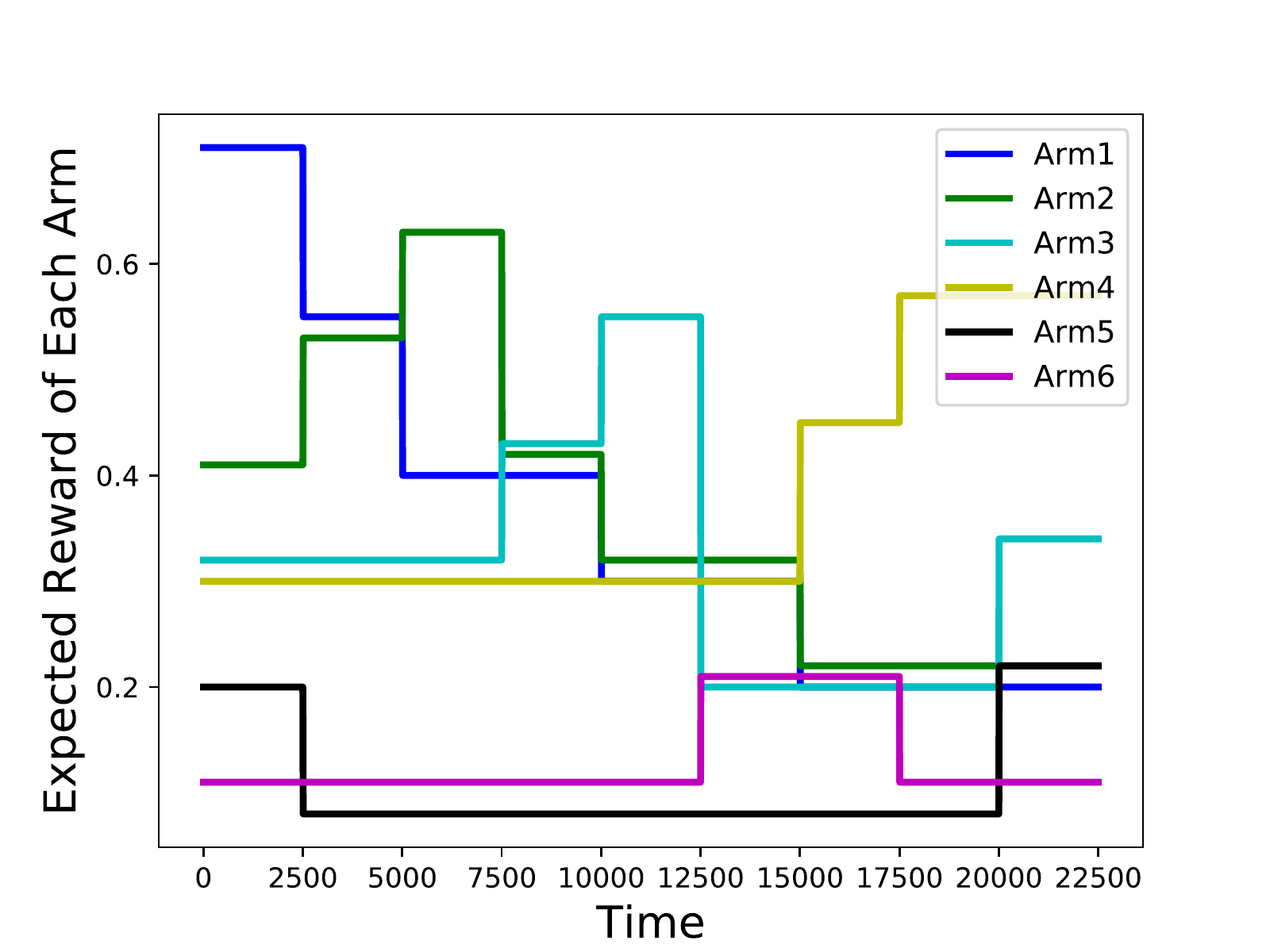}\
    \includegraphics[width=0.3\textwidth]{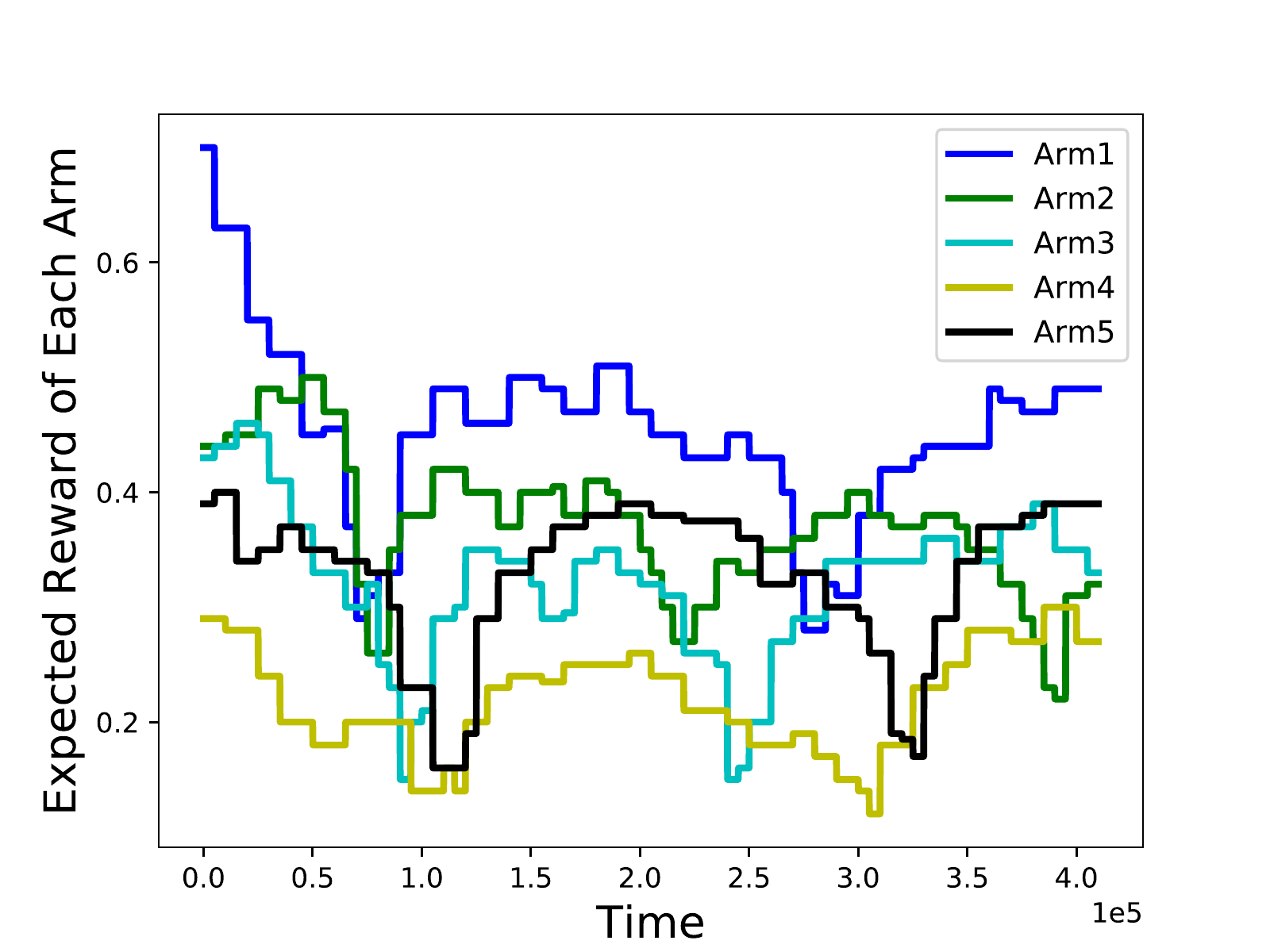}
    \caption{Reward distributions of base arms for all experiments. Left: synthetic data. Middle: Yahoo! experiment 1. Right: Yahoo! experiment 2.}
    \label{fig:reward_dist}
\end{figure*}\\
We also run all algorithms on a much harder problem,  by extracting the experiment data from Yahoo! dataset following~\citet{liu2018change}. Again, to make the experiment nontrivial, we enlarge the click rate of each base arm by $10$ times but the time horizon $T=410000$ is kept. The result is summarized in Figure~\ref{fig:yahoo2}.
\begin{figure}[!ht]
    \centering
    \includegraphics[width=0.7\columnwidth]{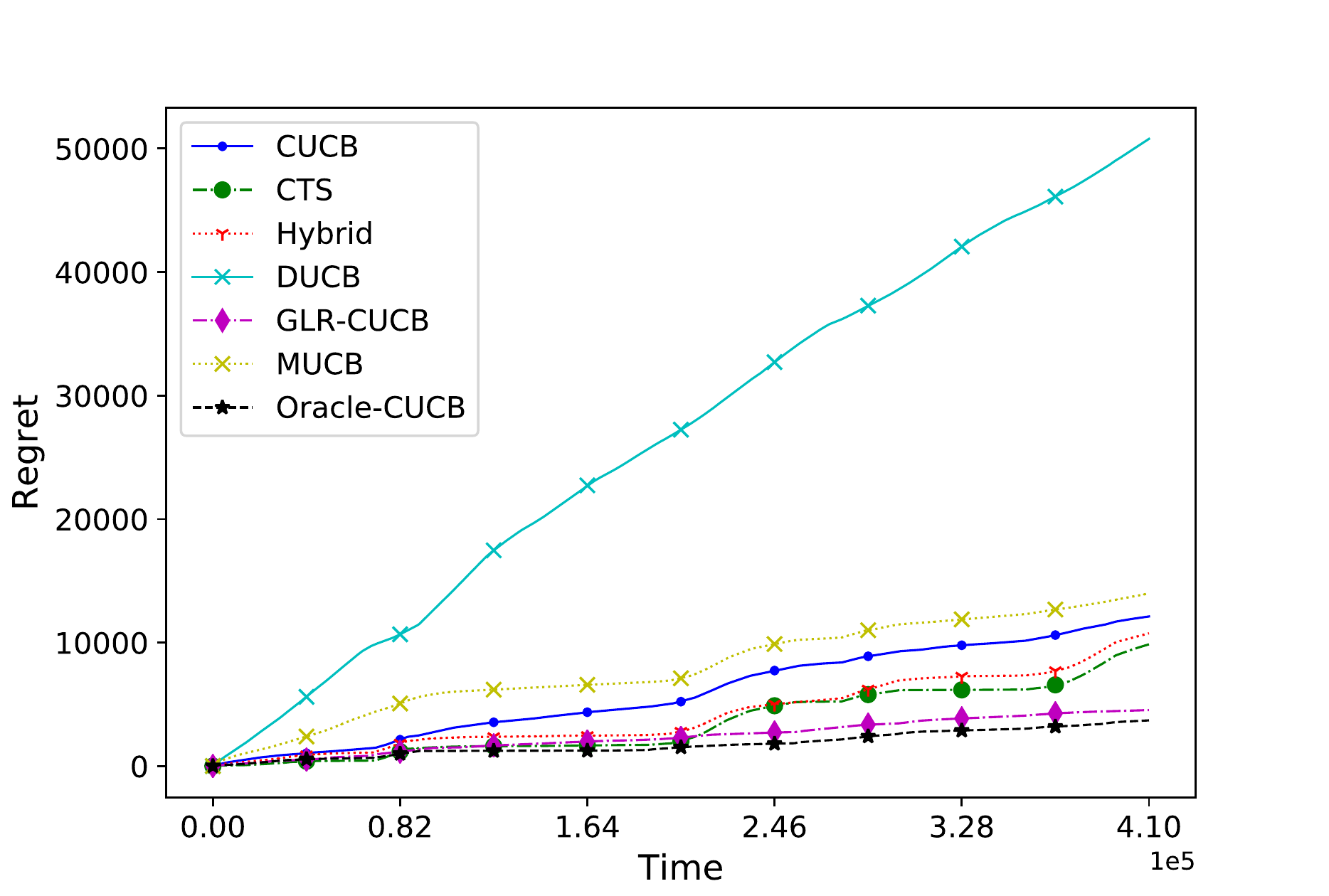}
    \caption{Expected cumulative regret for different algorithms on Yahoo! experiment2.}
    \label{fig:yahoo2}
\end{figure}
Lastly, we provide the standard deviation of the last step of all algorithms for different experiments, as shown in Table~\ref{table:std}.
\begin{table*}[htbp]
\centering
\begin{tabular}{|c||c|c|c|c|c|c|c|c|}

\hline\hline
                   & CUCB  & CTS     & Hybrid & DUCB & GLR-CUCB & LR-GLR-CUCB & MUCB   & Oracle-CUCB \\ \hline
Synthetic Dataset & 241.08 & 351.12 & 278.82 & 14.08 & 37.96 & 73.30 & 171.45 & 25.54       \\ \hline
Yahoo! Experiment 1 & 510.20 & 513.41 & 826.01 & 25.44 & 62.76 & 63.37  & 202.44 & 35.08 \\  \hline
Yahoo! Experiment 2 & 563.54 & 562.24 & 1189.17 & 158.27 & 517.13 & NA & 1427.28 & 160.76     \\ \hline
\end{tabular}
\caption{Standard deviations of all algorithms for experiments on synthetic and Yahoo! datasets}
\label{table:std}
\end{table*}

\subsection{Implementation Details}
\label{subsec:imp_detail}
To guarantee the reproducibility of our experiments, we include our choices of parameters of all algorithms for different experiments. Only \texttt{CUCB}, \texttt{CTS} and \texttt{Oracle-CUCB} do not have any parameters to be tuned.\\
\texttt{Hybrid}: $\gamma$ is set to be 1 for all three experiments. \\
\texttt{DUCB}: $\gamma=1-\sqrt{1/T}/4$ and $\xi=0.5$ for all experiments on synthetic and Yahoo! datasets.\\
\texttt{MUCB}: $w=150,500,1500$ for synthetic experiment, Yahoo! experiment 1 and Yahoo! experiment 2, respectively; $b=\sqrt{\frac{w}{2}\log(2|\mathcal{F}|T^2)}, \gamma = 0.05\sqrt{(N-1)|\mathcal{F}|(2b+3\sqrt{w})/(2T)}$ for all all experiments on synthetic and Yahoo! datasets.\\
\texttt{GLR-CUCB}: $\delta = \frac{20}{T}$ for experiment on the synthetic dataset, $\delta = \frac{50}{T}$ for Yahoo! experiment 1, and $\delta = \frac{70}{T}$ for Yahoo! experiment 2. $p=0.05\sqrt{(N-1)\log{T}/T}$ for all experiments on synthetic and Yahoo! datasets.\\
\end{document}